\documentclass{article}
\PassOptionsToPackage{numbers, compress}{natbib}
\usepackage[preprint,main]{neurips_2026}

\usepackage[utf8]{inputenc} % allow utf-8 input
\usepackage[T1]{fontenc}    % use 8-bit T1 fonts
\usepackage{nicefrac}       % compact symbols for 1/2, etc.
\usepackage{microtype}      % microtypography
\usepackage{xcolor}         % colors
\usepackage{amsmath}
\usepackage{amssymb}
\usepackage{mathtools}
\usepackage{amsthm}
\usepackage{algorithm}
\usepackage{algorithmic}
\usepackage{tikz}
\usepackage{enumitem}
\usepackage{url}    
\usepackage{booktabs}
\usepackage{amsfonts}
\usepackage{subcaption}
\usepackage{multirow}
\usepackage{wrapfig}
\usepackage{hyperref}                        
\usepackage[capitalize,noabbrev]{cleveref}  % after hyperref  

\newtheorem{definition}{Definition}
\newtheorem{theorem}{Theorem}
\newtheorem{proposition}{Proposition}
\newtheorem{corollary}{Corollary}

\newcommand{\Gach}{\mathcal{G}}
\newcommand{\Lib}{\mathcal{L}}
\newcommand{\Cand}{\Sigma}
\newcommand{\cost}{\mathcal{C}}
\newcommand{\score}{\mathcal{F}}

\newcommand{\ind}{\mathbb{I}}

\definecolor{figorange}{HTML}{F09828}
\definecolor{figblue}{HTML}{4FA8D8}
\definecolor{figpurple}{HTML}{803088}

\crefname{section}{Sec.}{Secs.}
\crefname{equation}{Eq.}{Eqs.}
\crefname{figure}{Fig.}{Figs.}
\crefname{table}{Tab.}{Tabs.}

\DeclareMathOperator{\repemp}{RepEmp}
\DeclareMathOperator{\envemp}{EnvEmp}

\usepackage[most]{tcolorbox}
\newtcolorbox{promptbox}[1][]{
  colback=gray!8,
  colframe=gray!50,
  boxrule=0.5pt,
  arc=2pt,
  fonttitle=\footnotesize\bfseries\sffamily,
  coltitle=black,
  title=#1,
  fontupper=\footnotesize\ttfamily,
  breakable,
  enhanced,
  attach boxed title to top left={xshift=8pt, yshift=-8pt},
  boxed title style={colback=white, colframe=gray!50, arc=2pt, boxrule=0.5pt}
}
\newcommand{\promptvar}[1]{\textcolor{purple!70!black}{\{\{#1\}\}}}

\title{What Is Worth Representing? Representational Empowerment for Continual Model Construction}

\author{
  Fei Dai\thanks{These authors contributed equally to this work.} \\
  University of California, Berkeley\\
  \texttt{f\_d@berkeley.edu} \\
  \And
  Hanqi Zhou\footnotemark[1] \\
  University of Tübingen, TU Darmstadt \\
  \texttt{hanqi.zhou@uni-tuebingen.de} \\
  \And 
  Alison Gopnik\\
  University of California, Berkeley\\
  \texttt{gopnik@berkeley.edu} \\
  \And
  Charley Wu\\
  TU Darmstadt\\
  \texttt{charley.wu@tu-darmstadt.de} \\
}

\begin{document}
\maketitle

\begin{abstract}
The first problem of modeling the world is not just estimating the right parameters or causal structure, but deciding \emph{what should be represented at all}. 
We frame this problem as \emph{continual model construction}: an agent maintains an environment-specific model $M$ of an inaccessible world $W$ and curates a persistent library $\Lib$ of reusable representational elements across environments. 
We propose \emph{Representational Empowerment} ($\repemp$) to score candidate elements by how much they expand the agent's future capacity to model and plan, complementing the classic definition of empowerment, but redefined as control over internal representations instead of external states. 
We realize the framework as a hierarchical Curator-Actor architecture and test it across three experiments. 
In a closed-vocabulary causal-learning task, human participants construct causal models at varying abstraction granularities to maximize goal reachability rather than fidelity to the world---a signature better predicted by $\repemp$ over information-gain alternatives. 
Matched simulations reveal $\repemp$-guided construction contributes more than exploration to sufficient structure recovery and cross-task transfer.
Finally, in an open-vocabulary planning domain, an LLM-augmented Curator builds more compact symbolic libraries, which also generalize better than baselines. Ablating $\repemp$ eliminates these benefits. 
Together, these results identify $\repemp$ as a key principle for continual model construction: deciding what to build, retain, and reuse under bounded resources. 
\end{abstract}

\section{Introduction}
% intuitive example to motivate the story
There does not exist a complete, ready-made vocabulary for modeling our observations in the world.
Concepts such as force, mass, and acceleration in physics were \emph{constructed} to explain, predict, and intervene in the world. 
The resulting vocabulary (e.g., Newton's laws), is not the world itself, but a compressed representation sufficient for prediction and control. 
That vocabulary is also \emph{curated} over centuries~\citep{fogarty2015cultural}: Aristotelian categories gave way to Newtonian mechanics, which gave way to relativity. Each transition was driven by phenomena the old vocabulary could not support, while abstractions that continued to support reasoning in new domains persisted.

% connect to cogsci literature and propose the main question
Human problem-solving operates on similar principles at the smaller developmental scale. A child learning to cook does not build a complete chemical model of food; she constructs just enough (e.g., ``bigger food needs to cook longer'') to achieve her goals, retaining what transferred successfully and discarding what was specific to last Tuesday's dinner. 
This is a loop that cognitive science has long described: people plan with simplified, task-specific models rather than reconstructing the full complexity of the world~\citep{ho2022people,chen2026just,nagy2025analogy,wong2025modeling}, consistent with resource-rational accounts of bounded agents~\citep{lieder2020resource,gershman2015computational, mantiuk2025curiosity, zhou2026pathdependent}, and with children entertaining simplified causal hypotheses rather than full reconstructions~\citep{gopnik2012reconstructing}. 
Yet, both examples raise a question that precedes parameter estimation or structure learning: \emph{which representational elements should an agent build, retain, and reuse?}

% a brief literature review (since we move the complete literature review to appendix) 
Existing work addresses important parts of this problem without making model construction itself the main object of optimization. Continual learning seeks to preserve knowledge across tasks~\citep{kirkpatrick2017overcoming,serra2018overcoming,rolnick2019experience}; world-model and causal learning estimate structure within a supplied state or variable space~\citep{hafner2025mastering,ha2018recurrent,bramley2017formalizing,ke2021systematic}; and LLM-based agents leverage symbolic knowledge with little pressure to select what is kept~\citep{tang2024worldcoder,wang2023voyager}. 
These approaches motivate but do not yet provide a general criterion for deciding which candidate representational elements should be retained under resource limitations (\cref{app:related}).

\begin{figure}[t]
    \centering
    \includegraphics[width=\textwidth]{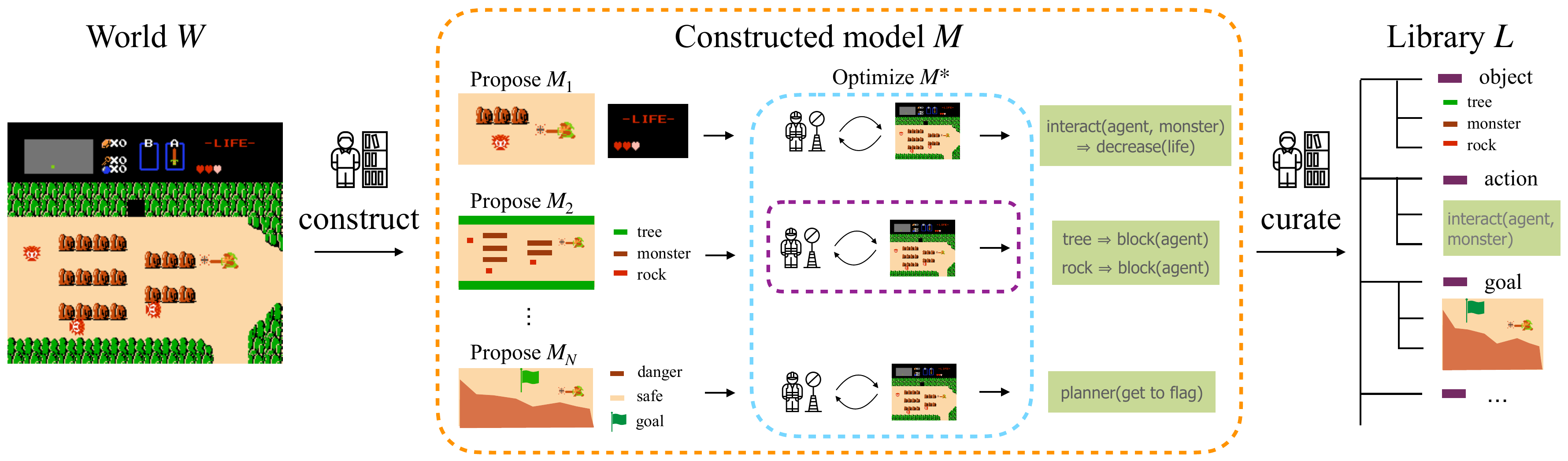}
    \caption{\textbf{Continual model construction as a $W \rightarrow M \rightarrow \Lib$ pipeline. }
    The agent is unable to recover the full world model $W$; instead it \textcolor{figorange}{constructs candidate models $M_1,\dots,M_N$}, optimizing each against goals and curating reusable elements into a persistent library $\Lib$. Existing paradigms operate inside this loop without choosing $M$:  \mbox{\textcolor{figpurple}{world-model learning}} optimizes a single monolithic $M$ to all observation data, and the \mbox{\textcolor{figblue}{continual and meta-learning}} reuses $M$ across environments. Our framework treats the choice of $M$ and the curation of $\Lib$ as the primary objects of optimization. }
    \vspace{-1.5em}
    \label{fig:pipeline}
\end{figure}

% mention our proposal intuitively; should leave the math to sec2 or sec3
Answering this question requires separating three objects that formal treatments often conflate (\cref{fig:pipeline}): an environment's inaccessible generative process $W$, the task model $M$ constructed from finite interaction, and a persistent library $\Lib$ of reusable elements.
The resulting $W\to M\to\Lib$ pipeline makes two decisions explicit: what should enter the current representation, and what should survive beyond it.
Within an environment, \emph{construction} changes $M$ so that the agent can model and plan for its current goals; across environments, \emph{curation} determines which elements remain available in $\Lib$ for future tasks. 
To guide both decisions, we introduce \emph{representational empowerment} ($\repemp$) \citep{zhou2025agent} as an information-theoretic measure of how new representations expand the agent’s future capacity to model and plan.
\Cref{sec:setup} formalizes this criterion, and \cref{sec:method} shows how a Curator-Actor architecture evaluates them directly in enumerable settings and approximates them when representations must be invented online.

\textbf{Contributions.}
(1) We formulate continual model construction as a $W \to M \to \Lib$ pipeline where agents construct environment-specific models and curate reusable elements across tasks (\cref{sec:setup}).
(2) We introduce \emph{representational empowerment} as a model construction criterion, and realize it in a Curator-Actor architecture that proposes, scores, and curates elements through planning and execution (\cref{sec:method}).
(3) In closed-vocabulary settings, where the candidate elements form a fixed set of representation granularities (\cref{sec:closed-vocab}), human choices are best explained by empowerment; matched simulations then show that $\repemp$-based representation selection causally improves cross-task transfer (\cref{sec:closed-vocab}).
(4) In an open vocabulary, an LLM-augmented Curator invents predicates and operators online and uses execution-grounded evidence to build compact Planning Domain Definition Language (PDDL)~\citep{aeronautiques1998pddl} libraries. Against PPO, Motif~\citep{klissarov2023motif}, and WorldCoder~\citep{tang2024worldcoder}, the curator demonstrates better cross-task transfer and library compactness.

\section{Representational Empowerment for Model Construction}\label{sec:setup}
We frame continual model construction as decisions about what enters a task model and what persists in a library. 
This section motivates our framework: 
why \emph{three} modeling objects should be considered rather than one, 
why \emph{empowerment} is the right criterion for what those objects should contain, 
and why empowerment must be measured over \emph{internal} representations.
\cref{sec:method} then turns this framework into a concrete hierarchical agent. Formal properties and proofs are deferred to \cref{app:theory}.

\textbf{Three objects in the $W \to M \to \Lib$ pipeline.}
World-model learning typically treats the environment $W$ as the object to be learned, seeking a model that captures its dynamics from experience~\citep{hafner2025mastering}.
For a resource-bounded agent, however, representing all of $W$ is neither feasible nor necessary. It must construct a limited task model $M$ and decide which elements to carry forward in a reusable library $\Lib$~\citep{wong2025modeling,nagy2025analogy}.
Concretely, the agent encounters environments $e_1,e_2,\ldots\sim\mathcal E$, each generated by an inaccessible process $W_e$. It assembles a task model $M_e$ from candidate elements $\sigma\in\Cand$, where an element may be an observation granularity, causal relation, predicate, or operator schema. A persistent library $\Lib$ carries selected elements across environments. Thus, \emph{construction} modifies $M_e$ within an environment, whereas \emph{curation} modifies $\Lib$ across environments. 

\begin{wrapfigure}{r}{0.44\textwidth}
    \vspace{-1em}
    \centering
    \includegraphics[width=0.9\linewidth]{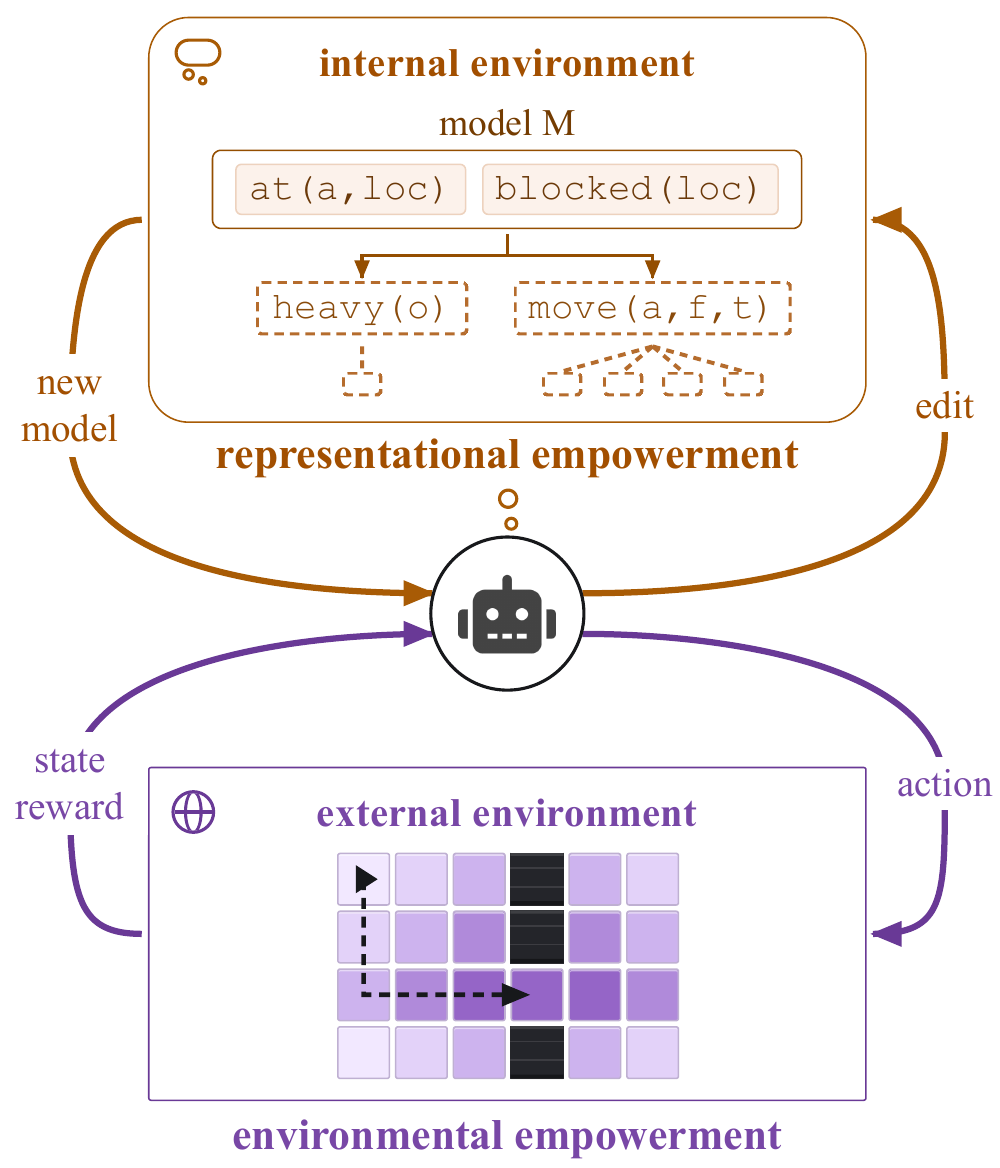}
    \caption{\textbf{Two forms of empowerment.} Representational empowerment ($\repemp$) describes how an agent edits its own model $M$, encouraging representations that support reaching future models (dashed boxes).
    Environmental empowerment ($\envemp$) is the traditional definition, describing how an agent acts on the external environment and encourages reaching future states (shading).}
    \label{fig:emp-schematic}
    \vspace{-1em}
\end{wrapfigure}

\textbf{What to represent?}
Once $M_e$ and $\Lib$ are themselves objects of choice, the central problem is how to evaluate a candidate representational change: which changes expand the agent's modeling capacity and which merely add irrelevant detail?
\emph{Information gain}~\citep{lindley1956measure,loewenstein1994psychology, nelson2005finding} rewards uncertainty reduction about $W_e$, whether or not the resolved distinction changes what the agent can represent or use.
\emph{Environmental empowerment}~\citep{klyubin2005empowerment,salge2014empowerment,modirshanechi2025testing,lidayan2025intrinsically} instead measures controllable future diversity of states within $e$: it treats the environment as an information channel from action sequences to future states, whose capacity $\envemp_T(s;e)=\sup_{p(a_{1:T})}\,I(A_{1:T};S_T \mid S_0=s,e)$ is high when many states are reachable from $s$ and the agent can reliably steer to a chosen target.
This provides the right structural idea about controllable future diversity, but is mapped to external states, leaving the agent's representation fixed (see \cref{app:alternatives} for a detailed contrast).

\emph{Representational empowerment} instead makes representational change the channel itself: edits to the agent's representation are the inputs and future representations are the outputs~\citep{zhou2025agent}.
Let $Z_t=(M_t,\Lib_t)$ denote the agent's representational state and $\Omega$ a finite alphabet of edits to that state (e.g., \textsc{add}, \textsc{remove}, or \textsc{refine}; \cref{app:rep-channel}).
Over a horizon $T$, an edit sequence $U\in\Omega^T$ induces through $K_T$ a distribution over future representational states $Z_T$.
The capacity of this channel defines representational empowerment:
\begin{equation}\label{eq:exact-repemp}
   \repemp = \mathcal{R}_T(z)
   \;=\;\sup_{\rho\in\Delta(\Omega^T)}\,
   I_{\rho K_T}\!\big(U;\,Z_T \mid Z_0=z\big).
\end{equation}
Thus, $\repemp$ is high when edits reliably reach diverse future representations. 
Eq.~\ref{eq:exact-repemp} defines the exact criterion, though evaluating it empirically requires an observable way to compare the internal representations it distinguishes; we address this next by projecting the channel onto goal achievement.

\textbf{Grounding the criterion in the environment.}
Exact representational empowerment distinguishes future internal states $Z_T$, including representations that differ syntactically but support the same behavior.
For empirical evaluation, we observe a task model through its goal-achievement signature $\Gamma_e(M)=(\,\ind[\,g\text{ is achieved by planning with }M\text{ and executing in }e\,] \,)_{g\in\Gach_e}$.
Applying $\Gamma_e$ to the channel's outputs grounds it in goal outcomes in $e$:
\begin{equation}\label{eq:goal-proj-repemp}
   \mathcal{R}^{\Gamma}_T(z;e)
   =\sup_{\rho\in\Delta(\Omega^T)}
   I_{\rho K_T}\!\big(U;\Gamma_e(Z_T)\mid Z_0=z\big).
\end{equation}
This projection is a measurement of representational empowerment rather than an alternative definition: it preserves exactly those distinctions in $Z_T$ that are visible through goal achievement.
When the projected channel is enumerable, we evaluate it directly; otherwise, \cref{sec:method} approximates its relevant consequences using finite probes and execution feedback.
Its formal relation to exact representational empowerment is given in \cref{app:theory}.

\section{The Curator-Actor Architecture}\label{sec:method}

Continual model construction operates on two timescales: representational commitments persist across environments, whereas interactions with the environment unfold within a single task. 
We realize the agent with a hierarchical architecture (\cref{fig:experiment-overview}a): a slow, high-level \emph{Curator} $\mathcal C$ decides what to build and retain (i.e., editing the task model $M_t$ and the library $\Lib$), while a fast, low-level \emph{Actor} $\mathcal A$ plans and acts under the model supplied by the Curator. 
Each level of the architecture only depends on what it needs. The Actor depends only on the current model, not the history of edits. The Curator depends on execution outcomes, not the actions that achieved them.
We present the architecture at the level of detail the experiments require, and give its full specification in \cref{app:theory,app:model-construction} with the loop itself in \cref{alg:cmc}.

\begin{figure}[t]
    \centering
    \includegraphics[width=.98\textwidth]{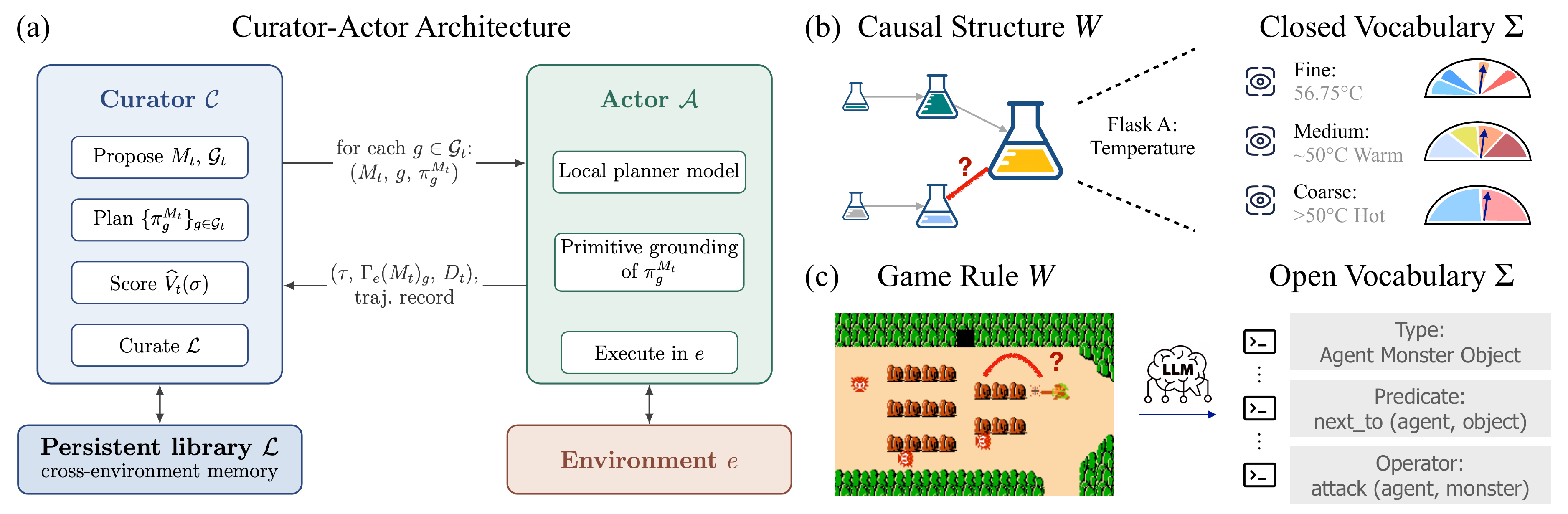}
    \vspace{-.5em}
    \caption{\textbf{Curator-Actor architecture across experiments.}
    \textbf{(a)} The Curator proposes and scores model edits; the Actor grounds and executes their plans, returning execution evidence that determines what is retained in the persistent library.
    \textbf{(b)} Closed vocabulary (Exp.~1-2): $\Cand$ is a fixed set of model granularities/variables/properties.
    \textbf{(c)} Open vocabulary (Exp.~3): $\Cand$ is an LLM-proposed space of PDDL types, predicates, and operators; curated elements enter $\Lib$ and seed $M$ on the next task.}
    \label{fig:experiment-overview}
    \vspace{-1.5em}
\end{figure}

\textbf{Propose and probe (Curator).}
At iteration $t$, the Curator receives the library $\Lib$, observations from environment $e$, and a trajectory buffer $\mathcal D_t$. It assembles $M_t$ from relevant elements of $\Lib$ and new candidate elements from $\Cand$, and constructs a probe-goal set $\mathcal G_t$ that puts those candidates to the test: 
\emph{positive probes} test goals a candidate could newly enable (gained capabilities), whereas \emph{regression probes} test goals the previous model already supported (preserved capabilities). 
The Curator then plans for each probe goal under $M_t$ and hands the plans to the Actor.

\textbf{Execute (Actor).}
For each probe goal $g$, the Actor receives the model, the goal, and its plan $(M_t,g,\pi_g^{M_t})$. Because the plan is expressed in the abstract operators of $M_t$, the Actor grounds it into the primitive actions available in $e$ and executes it. 
It reports back the realized trajectory, whether the goal was achieved ($\Gamma_e(M_t)_g$), and any operator whose predicted effects failed to hold. 
Across $\mathcal G_t$ these outcomes give the goal-achievement signature $\Gamma_e(M_t)$, the execution evidence that separates a merely plausible symbolic proposal from a representation that supports successful behavior.

\textbf{Score and promote (Curator).}
With the probe outcomes in hand, the Curator evaluates each candidate by its effect on the grounded channel. When the candidates and goal outcomes can be enumerated (Exps.~1-2), the score is computed exactly. In the open-vocabulary setting (Exp.~3), neither the candidate space nor the space of relevant goals can be enumerated. Only there, we use the finite-sample approximation
\begin{equation}\label{eq:expected-action-value}
   \widehat V_t(\sigma)
   =q_t(\sigma)
   \sum_{g\in W_t(\sigma)}
   \left[\widehat p_t(g\mid M_t\cup\{\sigma\})
      -\widehat p_t(g\mid M_t)\right]
   -\lambda c_\sigma,
\end{equation}
where $q_t(\sigma)$ estimates the validity of candidate $\sigma$, the witness set $W_t(\sigma)$ contains sampled goals that are currently blocked but could be unlocked by $\sigma$, $\widehat p_t$ is the empirical achievement rate returned by the Actor, and $\lambda>0$ weights the representational cost $c_\sigma$. 
The Curator accepts $\sigma$ into $M_t$ when $\widehat V_t(\sigma)>0$ and its probe successes are reliable, and otherwise retains the previous model. This score approximates \cref{eq:goal-proj-repemp}, whose capacity comes entirely from edits that change goal outcomes; $\widehat V_t$ measures those changes directly on sampled probes.

The three steps above adapt $M_t$ within a single environment, while curation operates at the slower timescale of the environment sequence $e_1,e_2,\ldots$ (\cref{sec:setup}). Before moving on to the next environment, the Curator re-evaluates each retained element conditional on the rest of the library and the accumulated evidence. Removing elements that are unreliable, redundant, or narrowly task-specific compresses the library, and the resulting $\Lib$ seeds model construction in the next environment.

\textbf{Experimental regimes.}
The framework makes three claims that require different kinds of evidence:  (i) $\repemp$ \emph{describes} how natural agents construct models, (ii) optimizing $\repemp$ \emph{improves} performance, and (iii) $\repemp$ \emph{scales} to settings where the vocabulary itself must be invented. 
We therefore instantiate the Curator-Actor architecture across two regimes (\cref{fig:experiment-overview}b-c). 
In the \emph{closed-vocabulary} regime, the vocabulary is a fixed menu of representation granularities, so every criterion for selecting representational elements ($\repemp$, $\envemp$, information gain) can be scored exactly on the same alternatives. Exp.~1 provides correlational evidence by fitting these criteria to human choices, 
while Exp.~2 provides causal evidence by implementing them in matched simulated agents.
In the \emph{open-vocabulary} regime, an LLM proposes PDDL predicates and operators online, so models must be invented rather than chosen from a menu and $\repemp$ is only available through the finite-sample score of \cref{eq:expected-action-value}. Exp.~3 tests whether empowerment-guided curation survives such recursive dynamics and whether it builds a compact, transferable library.
The experiments thus progress from human alignment, to causal mechanism, to open-vocabulary realization.

\section{Closed-Vocabulary Causal Learning}\label{sec:closed-vocab}

With the framework and architecture in place, we test two implications: whether $\repemp$ predicts human representation choices, and whether optimizing $\repemp$ improves structure learning and generalization. 
In a closed-vocabulary setting, the candidate representations are finite and enumerable, so we can compute $\repemp$, $\envemp$, information gain exactly to compare them.
Specifically, we use a causal-discovery environment in which $\Cand$ is a fixed set of representation granularities. 
Exp.~1 (\cref{sec:human}) tests the framework's predictions against human representation and intervention choices. 
Exp.~2 (\cref{sec:sim}) implements each criterion (representation choice and intervention choice) with simulated agents and tests whether $\repemp$ captures optimal granularity under different conditions.

\textbf{Causal-discovery environment. }
Both experiments use a multivariable causal-discovery environment adapted from \citet{ke2021systematic}, presented as an alchemy task. The world $W$ is a structural equation model (SEM) whose latent variables are flasks, each carrying three observable properties (temperature, color, and weight). Learners never observe $W$ directly. Instead, they complete a sequence of tasks, each exposing a small subset of the flasks. Within a task, the learner explores by intervening to discover the causal rules among the visible flasks. Each step comprises two choices.
First, the learner selects an observation granularity $\ell\in\{\texttt{coarse},\texttt{medium},\texttt{fine}\}$. 
Second, the learner intervenes on a variable-property pair $(v,r)$, such as setting flask A's temperature to a chosen value. The intervention propagates through $W$, and its consequences are reported at the selected granularity. At the end of each task, the learner reports a causal graph and uses it in goal-directed tests, where each goal asks for a specific property to be brought to a target value (e.g., ``set flask B's weight to level 3''). 
Here $M_t$ is the causal model constructed from observations at $\ell$ (different granularities), and $\Lib$ contains structure retained across task windows. Finer observations distinguish more values but are noisier, so the granularity that best identifies goal-relevant structure need not be the finest (\cref{app:model-selection-env}).

\subsection{Experiment 1: Human causal learning}\label{sec:human}

\begin{figure}[t]
\centering
    \includegraphics[width=\textwidth]{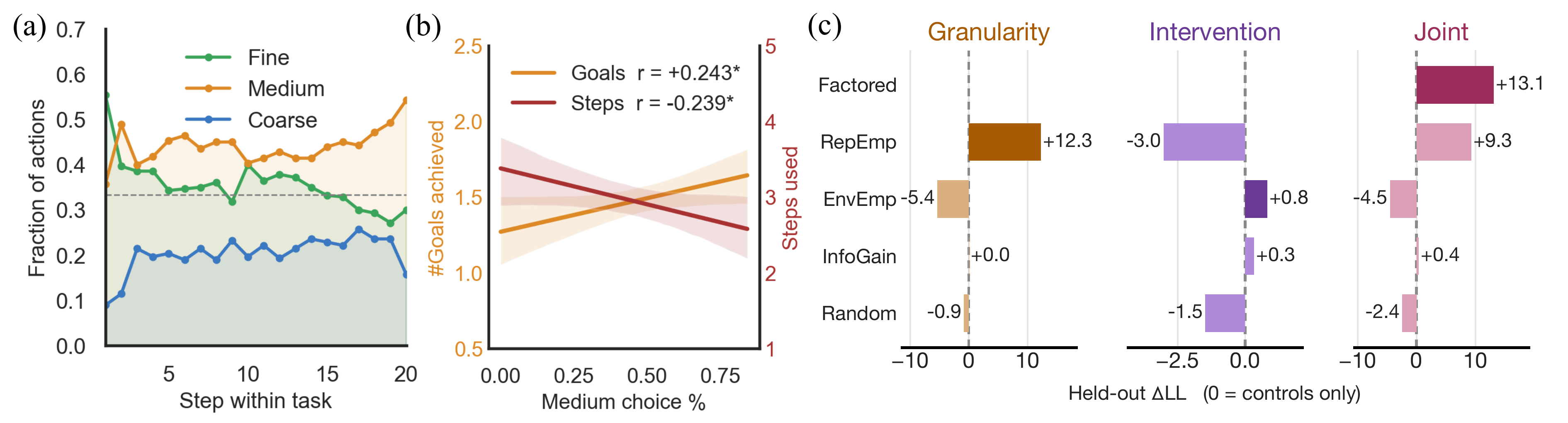}
    \caption{\textbf{Exp.~1 results.}
    \textbf{(a)} Fraction of actions at each observation granularity as a function of step within a task, averaged across participants. 
    \textbf{(b)} Across participants, greater use of \texttt{medium} was associated with more goals achieved and fewer steps to achieving a solution in the generalization environments. 
    \textbf{(c)} Held-out $\Delta$LL of each selection criterion (including random effects and control variables) for granularity choice, intervention choice, and both jointly. Zero is the control-only model and positive values indicate better performance.
    }
    \label{fig:human_results}
\vspace{-1.5em}
\end{figure}

Participants ($n{=}70$; see \cref{app:human}) completed four such tasks, choosing an observation granularity $\ell$ and a flask-property intervention $(v,r)$ at each step. 
We ask whether (A)~participant behavior is consistent with resource-rational model construction, (B)~whether this pays off in goal achievement, and (C)~if $\repemp$ provides the best account of human choices about the granularity of observations. 

\textbf{(A) Human choices are consistent with resource-rational model construction.}
Participants chose \texttt{medium} granularity on $43\%$ of decisions, compared with $38\%$ for \texttt{fine} and $19\%$ for \texttt{coarse} (\cref{fig:human_results}a). $\repemp$ also prefers \texttt{medium} granularity under this setting, since it provides sufficient resolution for achieving the probe goals without adding too much noise (see Exp.~2 for settings where $\repemp$ predicts other granularities).
Across tasks, choices converged on \texttt{medium} from both extremes, with \texttt{medium} rising from $35\%$ to $46\%$ ($t(69){=}2.83$, $p{=}.006$) and \texttt{coarse} \emph{falling} from $25\%$ to $17\%$ ($t(69){=}{-}2.86$, $p{=}.006$). This rules out a cost-only account, which predicts participants simply minimize representational cost (see \cref{app:human-behavior} for within-task dynamics). Cost minimization would consistently favor \texttt{coarse}, whereas participants moved away from it instead, progressively favoring the task-sufficient \texttt{medium} granularity (\cref{app:prop:nonstationary}). 
A second alternative is that \texttt{medium} is preferred because it is easier to process, with goal sufficiency playing no role. An effort account would predict that granularities differing in processing demand would differ in reaction time (RT), but median RTs were largely equivalent at $7.6$/$7.6$/$7.3$\,s for \texttt{coarse}/\texttt{medium}/\texttt{fine}, respectively, with the \texttt{fine}-\texttt{medium} contrast on log-RT returning a null result ($\beta{=}{-}0.003$, $p{=}.94$; \cref{app:human-behavior}).

\textbf{(B) Goal-sufficient compression covaries with performance.}
Greater \texttt{medium} use predicted more goals achieved ($r{=}.24$, $t(68){=}2.07$, $p{=}.042$) and in fewer steps ($r{=}{-}.24$, $p{=}.046$; \cref{fig:human_results}b), while greater \texttt{fine} use predicted the opposite ($r{=}{-}.29$, $p{=}.015$; $r{=}.31$, $p{=}.010$), and \texttt{coarse} use predicted neither ($p{>}.26$). 
A logistic mixed model over the $560$ binary goal outcomes, with a participant random intercept, confirmed that participants' overall \texttt{medium} use (between-participant component) was positively associated with goal achievement ($\beta{=}1.29$, $z{=}2.07$, $p{=}.039$). The task-specific deviation from each participant's own average (the within-participant component) was positive but unreliable ($\beta{=}.59$, $p{=}.31$). We therefore interpret this as an association rather than a causal effect; Exp.~2 provides the corresponding manipulation (\cref{app:human-methods}).

\textbf{(C) $\repemp$ best explains granularity choices.}
To compare various choice criteria, we reconstructed each participant's belief state online and scored all actions on the same observed trajectory. Utilities were evaluated using leave-one-task-out predictions in an hierarchical regression controlling for population and participant random effects, and with previous-choice stickiness, and intervention familiarity as control variables (\cref{app:human-action}). 
\Cref{fig:human_results}c reports $\Delta$LL as model improvement over a baseline containing only the control variables.
$\repemp$ provided the best predictions of granularity, while $\envemp$ provided the best predictions of intervention source. 
However, combining both $\repemp$ and $\envemp$ yields the largest joint increment across all pairings, with almost all of the factorized improvement in predictions coming from $\repemp$. 

\textbf{Exp.~1 summary.}
Across tasks, participants progressively compress from initially fine observations toward the goal-sufficient \texttt{medium} granularity. Greater \texttt{medium} use was associated with better goal performance. 
Controlling for individual preferences and stickiness, the factored $\repemp$-granularity $\times$ $\envemp$-intervention criteria achieved the highest joint fit, with $\repemp$ playing the largest role.
Next, Exp.~2 uses simulations to directly test how these criteria influence learning and generalization.

\begin{table}[t]
\centering
\begin{minipage}[t]{0.44\textwidth}
    \centering
    \caption{\textbf{Exp.~2: Graph recovery when \texttt{medium} is optimal.} Absolute F1 (mean, SEM; $150$ seed-task instances per split).}
    \label{tab:sim-headline}
    \scriptsize
    \vspace{8pt}
    \begin{tabular}{lcc}
        \toprule
        Condition & Train & Held-out \\
        \midrule
        Factored & $\mathbf{.720\ (.021)}$ & $\mathbf{.703\ (.021)}$ \\
        RepEmp & $.607\ (.022)$ & $.592\ (.020)$ \\
        InfoGain & $.304\ (.013)$ & $.383\ (.016)$ \\
        Random & $.307\ (.014)$ & $.365\ (.013)$ \\
        EnvEmp & $.195\ (.015)$ & $.120\ (.014)$ \\
        \bottomrule
    \end{tabular}
\end{minipage}\hfill
\begin{minipage}[t]{0.52\textwidth}
    \vspace{0pt}
    \centering
    \caption{\textbf{Exp.~2: One-factor-at-a-time ablations.} Held-out F1 (mean, SEM) when one decision is replaced and the other is held at the specified criterion. }
    \label{tab:sim-ablation}
    \scriptsize
    \begin{tabular}{llcc}
        \toprule
        Replaced & Criterion & Held-out F1 & $\Delta$ \\
        \midrule
        \multirow{3}{*}{Granularity}
          & $\repemp$ & $\mathbf{.703\ (.021)}$ & --- \\
          & InfoGain & $.425\ (.014)$ & $-.278$ \\
          & Random & $.376\ (.015)$ & $-.327$ \\
        \midrule
        \multirow{3}{*}{Intervention}
          & $\envemp$ & $\mathbf{.703\ (.021)}$ & --- \\
          & InfoGain & $.596\ (.020)$ & $-.107$ \\
          & Random & $.507\ (.021)$ & $-.196$ \\
        \bottomrule
    \end{tabular}
\end{minipage}
\vspace{-1em}
\end{table}

\subsection{Experiment 2: Mechanistic evaluation with agent-based simulations}\label{sec:sim}

We now use simulations to reproduce the causal-discovery problem with two changes the human experiment cannot make. Bayesian agents let us assign the selection criteria directly rather than inferring them from choices, while varying observation noise across granularities changes which granularity is optimal.
Each of $30$ seeds sampled a ten-variable DAG ($p_e{=}.25$) and generated a curriculum of five training and five held-out tasks. Each task exposed $4$-$7$ of the ten variables and retained about half of the preceding task's variables, so consecutive tasks shared structure and the beliefs formed in one task remained useful in the next. 
Agents were given $1{,}000$ interventions per task and carried their beliefs across task boundaries.
Each condition paired one granularity criterion with one intervention criterion ($\repemp$, $\envemp$, information gain, and random) (\cref{app:simulation}).

\textbf{(A) Representation selection drives performance.}
\Cref{tab:sim-headline} reports F1 score for different conditions. The factored condition pairs $\repemp$ granularity selection with $\envemp$ intervention selection, and achieves the best performance on both train and held-out splits.
Comparisons across criteria use the same task order for a fair comparison.
\Cref{tab:sim-ablation} clarifies the contributions of $\repemp$ and $\envemp$ by iteratively substituting each with the other criteria, using held-out F1 as a performance metric. 
Replacing $\repemp$ as the granularity criterion reduces performance close to the random floor for each alternative, whereas replacing $\envemp$ as the intervention criterion is less costly, even when interventions are chosen uniformly at random. 
Granularity selection therefore carries roughly three times the weight of intervention selection. 

\textbf{(B) Adaptation to different optimal regimes.} To test whether $\repemp$ can adapt to different optimal granularities, we manipulate observation noise separately for each granularity. $\repemp$ selects \texttt{medium} on $63\%$ of decisions when \texttt{medium} is the most reliable, and \texttt{fine} on $97\%$ of decisions when \texttt{fine} is the most reliable. 
In contrast, $\envemp$ prefers \texttt{fine} in every regime, because the state-diversity score grows with the number of bins, while InfoGain scores the three granularities roughly the same. 
Adapting to the granularity that is most reliable is therefore what separates $\repemp$ from the alternative criteria, and from fixing the granularity in advance (\cref{app:sim-identifiability}).

In the regime where \texttt{coarse} is optimal, $\repemp$ nevertheless selects \texttt{fine} on $91\%$ of decisions. With $1{,}000$ interventions, repeated \texttt{fine} observations can be averaged to reduce noise while preserving eight distinguishable value bins, compared with only two under \texttt{coarse}. When the budget falls to $20$, there are too few observations to average out this noise, and $\repemp$ instead selects \texttt{coarse} on $59\%$ of decisions (\cref{app:sim-identifiability}). 

\textbf{Exp.~2 summary.}
Exp.~2 establishes that representation selection causally drives graph recovery. Manipulating the granularity criterion changes held-out F1, which ablations attribute to granularity rather than intervention selection, while $\repemp$ adapts its granularity selection to different reliability contexts.

\section{Open-Vocabulary Symbolic Construction}\label{sec:open-vocab}

The closed-vocabulary experiments held the candidate set $\Cand$ fixed (\cref{sec:closed-vocab}). Here, we remove that assumption by allowing candidates to be generated online from interactions, expanding the question to whether the same selection criterion can build a transferable library when the vocabulary itself must be chosen. This open-vocabulary regime allows us to test the $M\to\Lib$ step of the pipeline under an unbounded candidate space.
We instantiate the Curator-Actor architecture with LLM-augmented components in two types of grid-worlds, where $M$ is a PDDL domain (i.e., typed predicates, operator schemas, or a type hierarchy) with grounding, and $\Lib$ carries promoted elements across environments.
When generating representations, we semantically define each candidate $\sigma$ as a typed predicate or operator schema rather than a representation granularity~\citep{silver2023predicate}. 
Since candidates can be invalid, the validity term $q_t$ in \cref{eq:expected-action-value} estimates whether the grounding code can fail on real observations. We also now use the cost term $\lambda c_\sigma$ to penalize retention, whereas it was previously inactive in the closed-vocabulary setting. 

We report three main findings. $\repemp$-scored curation (\textit{Curator} model) outperforms baselines on backward transfer and held-out generalization in both environments. Removing the empowerment criterion (\textit{Hoarder} model) collapses these gains, thus isolating it as a result of curation rather than proposal quality. As a result, the curator model requires substantially fewer LLM tokens than WorldCoder. 

\begin{figure}[t]
\centering
    \includegraphics[width=0.9\textwidth]{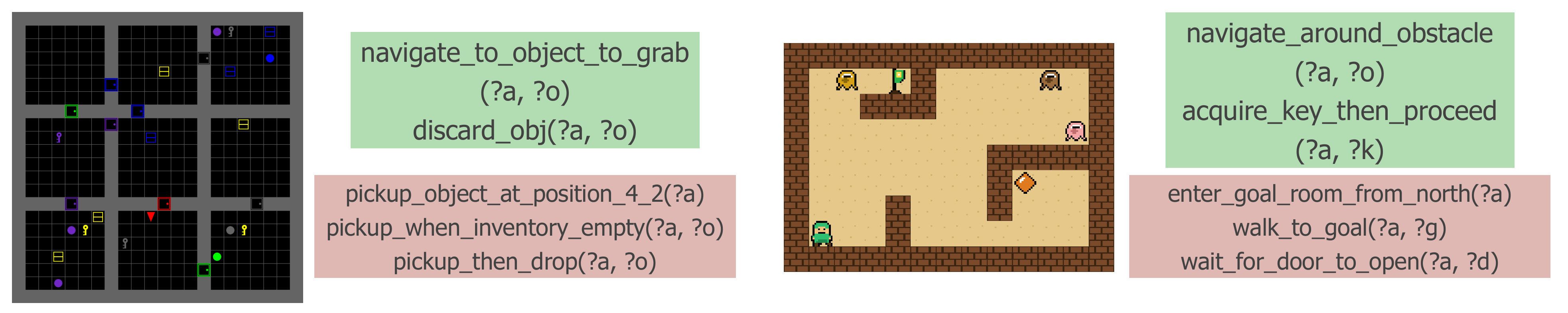}
    \vspace{-.5em}
    \caption{\textbf{Promoted vs.\ discarded operators in BabyAI (left) and Zelda (right).} Green = operators retained in $\mathcal{L}$ after passing the $\repemp$ threshold; red = candidates discarded despite being reliable on individual probes, because their witness sets are too narrow to earn persistence.}
    \label{fig:llm_example}
    \vspace{-1.5em}
\end{figure}

\subsection{Setup}\label{sec:exp-construction}

\textbf{Environments.}
We use two types of grid worlds. In \textit{BabyAI}~\citep{chevalier2018babyai}, training proceeds along a curriculum ordered by how many operators a solution requires:\texttt{GoToLocal} (1 operator), \texttt{Pickup} (3), \texttt{OpenRedDoor} (3), \texttt{UnlockLocal} (6+). Each environment introduces new operator requirements that compose those from earlier environments.  A held-out split tests compositional transfer to configurations not seen during training. 
In \textit{Zelda} we use the layout from EMPA ~\citep{tsividis2021human}, with the first three environments for training and the last two held out for testing. The environments naturally become more difficult as new kinds of monsters are introduced. Full details are in \cref{app:model-construction}.

\textbf{Evaluation.}
After training, $\Lib$ is frozen. We report \emph{forward transfer} ($\Lib$ frozen after environment $i$, evaluated on $i{+}1$), \emph{backward transfer} (the final $\Lib$ re-evaluated on each earlier training environment; \cref{app:curator-backward}), and \emph{held-out generalization} on the held-out environments.
We compare four baselines. PPO~\citep{schulman2017proximal} learns a primitive-action policy from rewards. Motif~\citep{klissarov2023motif} adds an LLM-shaped intrinsic reward but remains policy-level. WorldCoder~\citep{tang2024worldcoder} writes executable world-model code from traces and accumulates it without selective curation. The \emph{Hoarder} ablation uses the Curator's proposer, budgets, and reliability checks, but ``hoards'' each reliable proposal without being selective.
Interaction budgets are matched where applicable (\cref{app:baselines}): at matched LLM budgets, symbolic methods use about two orders of magnitude fewer environment steps than the policy baselines (\cref{app:curator-protocol}).

\subsection{Results}\label{sec:llm}

\begin{table}[t]
\centering
\scriptsize
\caption{\textbf{Open-vocabulary symbolic construction results.} \textbf{(a)} Transfer success across tasks within each environment (mean $\pm$ std over 3 seeds). \textbf{(b)} Training costs (total LLM tokens in millions) and library compactness (source code in KB averaged across the curriculum). See \cref{app:curator-results} for environment-specific results.}
\label{tab:llm-combined}
\begin{minipage}[t]{0.6\linewidth}
\centering
(a) Transfer success\\[2pt]
\begin{tabular}{lcccccc}
\toprule
& \multicolumn{3}{c}{BabyAI} & \multicolumn{3}{c}{Zelda} \\
\cmidrule(lr){2-4}\cmidrule(lr){5-7}
Method & Fwd. & Bwd. & H.O. & Fwd. & Bwd. & H.O. \\
\midrule
PPO          & .29{\tiny$\pm$.13} & .33{\tiny$\pm$.15} & .01{\tiny$\pm$.00} & .03{\tiny$\pm$.01} & .08{\tiny$\pm$.05} & .03{\tiny$\pm$.02} \\
Motif        & .04{\tiny$\pm$.03} & .23{\tiny$\pm$.10} & .12{\tiny$\pm$.05} & .19{\tiny$\pm$.03} & .40{\tiny$\pm$.05} & .18{\tiny$\pm$.02} \\
WorldCoder   & .32{\tiny$\pm$.11} & .40{\tiny$\pm$.10} & .35{\tiny$\pm$.10} & \textbf{.35}{\tiny$\pm$.02} & .36{\tiny$\pm$.02} & .38{\tiny$\pm$.02} \\
Hoarder & .20{\tiny$\pm$.11} & .24{\tiny$\pm$.10} & .03{\tiny$\pm$.01} & .07{\tiny$\pm$.04} & .12{\tiny$\pm$.02} & .06{\tiny$\pm$.00} \\
\textbf{Curator} & \textbf{.51}{\tiny$\pm$.19} & \textbf{.61}{\tiny$\pm$.07} & \textbf{.47}{\tiny$\pm$.13} & .24{\tiny$\pm$.07} & \textbf{.51}{\tiny$\pm$.06} & \textbf{.48}{\tiny$\pm$.01} \\
\bottomrule
\end{tabular}
    \end{minipage}\hfill
    \begin{minipage}[t]{0.34\linewidth}
    \centering
        (b) Training cost \& footprint\\[2pt]
        \begin{tabular}{lcc}
        \toprule
        Method & Tokens (M) & Library (KB) \\
        \midrule
        \multicolumn{3}{l}{BabyAI}\\
        WorldCoder      & 6.50 & 18.3 \\
        \textbf{Curator} & \textbf{5.14} & \textbf{8.0} \\
        \midrule
        \multicolumn{3}{l}{Zelda}\\
        WorldCoder      & 4.13 & 7.1 \\
        \textbf{Curator} & \textbf{2.94} & \textbf{4.4} \\
        \bottomrule
        \end{tabular}
    \end{minipage}
\vspace{-2em}
\end{table}

\textbf{(A) Curated libraries transfer backward and generalize to held-out environments.}
Table~\ref{tab:llm-combined}a reports success rates, where the Curator obtains the best backward and held-out transfer in both environments (see \cref{app:curator-results} for consistency across seeds). 
WorldCoder is competitive and exceeds the Curator on Zelda forward transfer, but failed on the backward and held-out columns. Thus, curated abstractions retroactively improve earlier environments and generalize to unseen ones, whereas typically accumulated code does not. 
In comparison, PPO and Motif transfer poorly because their policies are tied to the environment they were trained in. Motif's preference labels are at chance ($50.2\%$) without access to the explicit task specification, so its learned reward carries no signal (\cref{app:baselines}).
\Cref{fig:llm_example} shows which representational candidates were selected (green) or discarded (red) by $\repemp$. General schemas such as \texttt{navigate\_to\_object\_to\_grab} and \texttt{acquire\_key\_then\_proceed} are retained, while position- and context-specific candidates (\texttt{pickup\_object\_at\_position\_4\_2}, \texttt{enter\_goal\_room\_from\_north}) are discarded even though they succeed on their own probes.

\textbf{(B) Curation lowers LLM cost and keeps the library compact.}
Table~\ref{tab:llm-combined}b reports total LLM tokens consumed during training and the size of the final library.
WorldCoder's reward model grows monotonically with the number of environments (on BabyAI it triples from env.~0 to env.~3) because each new environment requires a new model, and prior branches cannot be deleted without breaking earlier goals. 
In contrast, the Curator selects sparingly and preserves reusable elements by carrying over typed predicates and operator schemas, such that new environments rebind the same handful of proven symbols to new objects. Its final library is $2.3\times$ (BabyAI) and $1.6\times$ (Zelda) smaller than WorldCoder's, and remains approximately constant in size across the curriculum (\cref{app:curator-compactness}).

\textbf{Failure case. }
On one BabyAI held-out environment, every method including the Curator was below $15\%$ success (\cref{fig:curator-heldout-babyai}). The environment requires a conjunction of binary object-object relations (``move the A next to the B and the C next to the D''), but every training goal the LLM proposed involved a single object or the agent's own state (holding, unlocked), never a relation between two objects. Even when the proposer generates a candidate \texttt{next-to} predicate, the witness set $W_t$ contains no compositional goal on which it could score.
Although $\repemp$ behaves correctly on the probes it was given, they did not cover the required goal structure. \cref{eq:expected-action-value} is in this sense conditional on the quality of probes in how well they capture future tasks (\cref{app:coverage-proxy}). Increasing the diversity of goal proposals~\citep{davidson2022creativity, perez2026cultural} is thus an important future direction.

\textbf{Exp.~3 summary.}
In both environments the vocabulary space is unbounded, such that the uncurated retention of acquired representations hinders performance, whether representations are symbolic proposals (Hoarder) or one synthesized reward function per task (WorldCoder). Instead, $\repemp$ as a curation criteria achieves better transfer in each domain, with lower LLM tokens and smaller libraries. And although absolute success rates remain modest, we read these results as preliminary, directional evidence for $\repemp$ as a model construction criterion rather than as a benchmark claim. More environments are needed to test the generality and probe the boundary cases.

\section{Discussion}

We asked which representational elements an agent should build, keep, and reuse. We propose \textit{representational empowerment}  $\repemp$ as the answer, which quantifies the value of a representational element by the extent to which it expands the plannable goal space per unit of representational cost. 
$\repemp$ provides an extension of the classic \textit{environmental empowerment} $\envemp$, by turning inwards to focus on the agent's choice of representations rather than control over the external environment. 
%Exp. 1 summary
We showed that human learners selectively chose the same granularity of representations as predicted by $\repemp$ in a causal learning task, leading to better performance. 
%Exp. 2 summary
Simulations revealed representation selection via $\repemp$ as the main normative factor, which more flexibly adapts between different optimal granularities than alterative models.
%Exp. 3 summary
Lastly, in the open-vocabulary regime, an LLM Curator scored the representations it had itself proposed by the same criterion, retaining a compact library that transferred better at lower token costs.
%and an LLM curator scoring its own inventions by the same 

\textbf{Limitations.} 
The human experiment used a fixed set of interventions and representational granularities, with evidence for $\repemp$ provided only via behavioral and model-based analyses, rather than process-level evidence. A mechanistic understanding of the normative benefits of $\repemp$ thus only comes from agent-based simulations.
Additionally, the open-vocabulary experiments used only two types of grid-worlds and should be read as a directional rather than a benchmarking result. The LLM curator can also only score what the proposer generates, so a narrow proposal distribution limits the effectiveness of $\repemp$. This is captured in a failure case, where $\repemp$ behaves correctly given the probes it saw, but the probes themselves did not sufficiently cover the relevant goal structure.

\textbf{Conclusion.} Continual learning, world-model learning, and causal discovery typically operate within a predefined and fixed representational vocabulary. Yet, the question of what representational vocabulary to build and maintain comes prior to such problems. Our results offer a candidate solution: agents should represent what expands their achievable goal space, then compress until only the minimum sufficient structure remains. Whether \emph{Representational Empowerment} generalizes more broadly is still an open question, but here, we provide a first account of representational dynamics.

\bibliographystyle{abbrvnat}
\bibliography{reference}

\newpage
\appendix
\onecolumn
\section{Related Work}\label{app:related}

The related works differ in which component of modeling they allow to change and which criterion governs that change. State-abstraction methods select task representations; symbolic and program-learning methods construct reusable vocabularies; empowerment and autotelic methods expand behavioral competence. Our framework connects these problems by treating both the environment-specific model $M$ and the persistent library $\Lib$ as objects of selection.

\textbf{Task-relative models and abstractions.}
State abstraction formalizes which distinctions a decision-making agent can discard. Bisimulation and approximate bisimulation preserve reward and transition structure~\citep{ferns2004metrics,castro2020scalable}; $\phi$-relevance and $Q^*$-irrelevance preserve policy or value information~\citep{li2006towards,abel2016near}; and PAC abstractions extend such guarantees across a distribution of tasks~\citep{abel2018state}. Related objectives include value equivalence for specified Bellman computations~\citep{grimm2020value}, successor features for transfer across reward functions~\citep{barreto2018transfer,lehnert2020successor}, goal-conditioned bisimulation~\citep{hansen2022bisimulation}, and skill-compatible neurosymbolic abstractions~\citep{ahmetoglu2026skill}. These methods generally optimize an abstraction of a supplied observation or state space under a specified equivalence or sufficiency criterion. Our setting extends representation selection to heterogeneous elements---including observation granularities, causal relations, predicates, and operator schemas---that can be added to or removed from $M$ and $\Lib$. Their value is determined by the goal-achievement distinctions they support and the cost of retaining them.

\textbf{Constructing and reusing symbolic vocabularies.}
Action-model induction and neuro-symbolic planning learn predicates and operator schemas that support planning~\citep{yang2007learning,cresswell2013acquiring,silver2023predicate,silver2022learning,chitnis2022learning}. LLM-based agents similarly construct PDDL domains, executable world models, or reusable skills from interaction~\citep{silver2022pddl,guan2023leveraging,tang2024worldcoder,wang2023voyager}. Program-induction systems such as EC and DreamCoder jointly solve tasks and learn abstractions that compress programs and improve subsequent search~\citep{ellis2021dreamcoder,zhou2026pathdependent}. These approaches establish that representational vocabularies can be learned. Their objectives usually emphasize planning performance, sample efficiency, or compression over a task corpus. We focus on retention where an element persists when execution shows that it improves goal achievement given the current library and its representational cost.

\textbf{Behavioral coverage and persistent competence.}
Classical empowerment measures controllable diversity through the mutual information between actions and successor states~\citep{salge2014empowerment,mohamed2015variational}. Skill discovery and option construction seek repertoires that span the controllable state space~\citep{gregor2016variational,eysenbach2018diversity,machado2017laplacian,park2024foundation}, while goal-conditioned and autotelic agents organize exploration around achieved-goal coverage~\citep{pitis2020maximum,colas2022autotelic}. These objectives select policies, skills, or goals within an existing representation. Representational empowerment instead evaluates changes to the vocabulary used to model and plan. Parameter-level continual-learning methods preserve knowledge through regularization, architectural isolation, or replay~\citep{kirkpatrick2017overcoming,serra2018overcoming,rolnick2019experience}; such mechanisms can stabilize components beneath a symbolic library, while library curation determines which reusable elements the system retains.

\textbf{Cognitive foundations.}
Our proposal draws on resource-rational accounts of representation choice~\citep{lieder2020resource}, constructivist models of causal learning~\citep{bramley2017formalizing}, and the ``child as scientist'' tradition~\citep{gopnik2012reconstructing}. It is also complementary to model-synthesis accounts in which distributed knowledge is assembled into a task-specific probabilistic model~\citep{wong2025modeling}. Representational empowerment supplies an operational hypothesis for valuing the candidate elements used in that assembly and deciding which should persist across tasks.

\section{Representational Empowerment}\label{app:theory}

This section distinguishes exact representational empowerment, its goal-projected channel, and the cost-regularized coverage proxy used in the open-vocabulary experiment. We show that projection retains only distinctions visible through the selected goal-achievement signature and that, under stated conditions, the proxy selects a minimum-cost goal-sufficient model. The proxy is an empirical decision rule, not a numerical estimate of channel capacity.

\subsection{Problem setup}\label{app:problem-setup}

The agent encounters environments $e_1,e_2,\ldots\sim\mathcal E$, each generated by an inaccessible process $W_e$. A candidate space $\Cand$ contains representational elements such as observation granularities, causal edges, predicates, and operator schemas. In environment $e$, the agent constructs a task model $M_e\subseteq\Cand$ and carries a library $\Lib\subseteq\Cand$ across environments. Construction edits $M_e$ within an environment; curation edits $\Lib$ across environments. We represent both as transitions of the internal state $Z=(M,\Lib)$.

\subsection{Representation-state channels}\label{app:rep-channel}

Let $\mathcal Z$ be the space of representation states and $\Omega$ a finite set of edits, such as \textsc{add}, \textsc{remove}, \textsc{refine}, \textsc{compose}, and \textsc{prune}. An edit sequence $u=\omega_{1:T}\in\Omega^T$ induces the kernel
\[
K_T(z'\mid z,u)=\Pr(Z_T=z'\mid Z_0=z,U=u),
\]
conditional on the current environment and evidence.

\begin{definition}[Exact representational empowerment]
\label{app:def:exact-repemp}
The $T$-step representational empowerment of state $z$ is
\[
\repemp=\mathcal R_T(z)
=\sup_{\pi\in\Delta(\Omega^T)}I_{\pi K_T}(U;Z_T\mid Z_0=z),
\]
where $U\sim\pi$ and $Z_T\sim K_T(\cdot\mid z,U)$.
\end{definition}

This is the internal analogue of environmental empowerment~\citep{abel2025plasticity,zhou2025agent}: edits, rather than actions, index controllable outcomes. Since
\[
I(U;Z_T\mid Z_0=z)
=H(Z_T\mid Z_0=z)-H(Z_T\mid U,Z_0=z),
\]
$\repemp$ requires future representations to be both diverse and reliably attributable to particular edits.

\subsection{Grounding by goal projection}\label{app:goal-projection}

Exact empowerment distinguishes all future representation states, including syntactic variants with identical behavioral consequences. We therefore observe a task model through its goal-achievement signature
\[
\Gamma_e(M)=
\bigl(\ind[g\text{ is achieved by planning with }M\text{ and executing in }e]\bigr)_{g\in\Gach_e}.
\]

\begin{definition}[Grounded representational empowerment]\label{app:def:grounded-repemp}
The grounded representational empowerment of state $z$ in environment $e$ is
\[
\mathcal R^{\Gamma}_T(z;e)
=\sup_{\pi\in\Delta(\Omega^T)}
I_{\pi K_T}\!\left(U;\Gamma_e(Z_T)\mid Z_0=z\right).
\]
\end{definition}

\begin{proposition}[Projection lower bound]\label{app:prop:projection}
For every deterministic projection $\Gamma_e$,
\[
\mathcal R^{\Gamma}_T(z;e)\leq\mathcal R_T(z).
\]
\end{proposition}
\begin{proof}
Conditional on $Z_0=z$, $U\to Z_T\to\Gamma_e(Z_T)$ is a Markov chain. The result follows from data processing for each $\pi$ and then taking the supremum.
\end{proof}

The inequality can be strict. Suppose $n$ edits deterministically produce distinct representations $z_1,\ldots,z_n$ but all have the same signature. The exact channel has capacity $\log n$, whereas the grounded channel is constant and has capacity zero. Thus the projection excludes purely syntactic diversity. When two representations have the same signature, the cost term introduced below selects the cheaper one.

\subsection{From channel capacity to the coverage proxy}\label{app:coverage-proxy}

Define the coverage of model $M$ by
\[
\score_e(M)=\sum_{g\in\Gach_e}\Gamma_e(M)_g.
\]
For a modeling action $a$ that yields observation $o$ and update $U(M_t,a,o)$, the expected cost-regularized gain is
\[
\widehat V_t(a)=
\mathbb E_{o\sim p_t(\cdot\mid a)}
\!\left[\score_e(U(M_t,a,o))-\score_e(M_t)\right]
-\lambda\Delta\cost_t(a).
\]

For an unvalidated candidate $\sigma$, let $q_t(\sigma)$ be its estimated validity and $W_t(\sigma)$ the goals currently blocked under $M_t$ that $\sigma$ could unlock. The open-vocabulary implementation uses
\[
\widehat V_t(\sigma)
\approx q_t(\sigma)
\sum_{g\in W_t(\sigma)}
\left[
\widehat p_t(g\mid M_t\cup\{\sigma\})
-\widehat p_t(g\mid M_t)
\right]
-\lambda c_\sigma,
\]
where $\widehat p_t$ is the Actor's empirical achievement rate. Restricting the sum to $W_t(\sigma)$ removes goals whose achievement bit cannot change.

This proxy preserves the local distinction relevant to the grounded channel. In a deterministic channel with inputs \textsc{keep} and $\textsc{add}(\sigma)$, capacity is zero exactly when adding $\sigma$ leaves $\Gamma_e$ unchanged; any changed achievement bit makes capacity positive. The proxy scores the magnitude and reliability of such changes rather than estimating capacity directly.

For a multi-element proposal, the exact coverage gain decomposes into conditional marginal gains under any ordering. The implementation approximates it by summing singleton gains, which is exact for independent candidates with disjoint witness sets. Redundancy can cause overestimation and complementarity can cause underestimation, so proposed bundles are verified jointly before acceptance. Higher-order representational synergy remains a limitation of the singleton approximation.

\subsection{Minimum-cost sufficiency for the coverage proxy}\label{app:minimal}

\begin{definition}[Goal sufficiency]
A model $M$ is goal-sufficient in environment $e$ if
\[
\score_e(M)=\score_e^\star:=\max_{M'\subseteq\Cand}\score_e(M').
\]
It is minimum-cost goal-sufficient if no goal-sufficient model has lower cost.
\end{definition}

Assume additive positive costs,
\[
\cost(M)=\sum_{\sigma\in M}c_\sigma,\qquad c_\sigma>0.
\]

\begin{theorem}[Minimum-cost goal sufficiency]\label{app:thm:minimal}
Assume $\Cand$ is finite and not every $M\subseteq\Cand$ is goal-sufficient. Let
\[
C_{\max}=\sum_{\sigma\in\Cand}c_\sigma,
\qquad
\Delta_e=
\min_{M:\,\score_e(M)<\score_e^\star}
\bigl(\score_e^\star-\score_e(M)\bigr)>0.
\]
If $0<\lambda<\Delta_e/C_{\max}$, every maximizer of
\[
\score_e(M)-\lambda\cost(M)
\]
is a minimum-cost goal-sufficient model.
\end{theorem}

\begin{proof}
Let $M$ be suboptimal in coverage and let $\widetilde M$ be any goal-sufficient model, which exists because $\Cand$ is finite. Their objective difference in favor of $\widetilde M$ is
\[
\score_e^\star-\score_e(M)
-\lambda\bigl(\cost(\widetilde M)-\cost(M)\bigr)
\geq\Delta_e-\lambda C_{\max}>0.
\]
Hence no coverage-suboptimal model maximizes the objective. Among goal-sufficient models the score is constant, so maximizing the objective minimizes cost.
\end{proof}

The theorem is a consistency result for the finite proxy under additive costs, exact signatures, and a sufficiently small $\lambda$; it does not claim that the agent knows $\Delta_e$ or that the proxy uniquely formalizes resource-rational construction.

\begin{definition}[Goal-irrelevant element]
An element $\sigma\in\Cand$ is goal-irrelevant in $e$ if
\[
\score_e(M\cup\{\sigma\})=\score_e(M)
\qquad\text{for all }M\subseteq\Cand\setminus\{\sigma\}.
\]
\end{definition}

\begin{corollary}[Pruning of goal-irrelevant elements]\label{app:cor:nuisance}
No maximizer in \cref{app:thm:minimal} contains a goal-irrelevant element.
\end{corollary}
\begin{proof}
Removing such an element preserves coverage and strictly reduces cost.
\end{proof}

In the causal experiments, a variable or edge outside every intervention-to-goal path is goal-irrelevant when the score depends only on intervention-to-goal reachability. Learning it may reduce uncertainty about $W_e$, but retaining it does not change the goal-achievement signature.

\subsection{Behavioral predictions for model construction}\label{app:predictions}

\begin{proposition}[Refinement during discovery, compression after sufficiency]
\label{app:prop:nonstationary}
There exists a two-stage construction problem in which \texttt{fine} has greater expected proxy value during discovery, whereas a coarser representation has greater value after the relevant structure is learned.
\end{proposition}

\begin{proof}
Let $M_f$ and $M_c$ be fine and coarse representations with $\cost(M_f)>\cost(M_c)$. During discovery, let a distinction available only under $M_f$ enable an expected score gain larger than its additional cost; then $M_f$ has greater proxy value. After the relevant relation is learned, let $\score_e(M_f)=\score_e(M_c)=\score_e^\star$; the proxy then prefers the cheaper $M_c$.
\end{proof}

The prediction is stage-dependent rather than a fixed preference for either resolution: fine representations can expose useful distinctions before the relevant structure is known, while the same goal coverage may later be retained more cheaply. This is the pattern tested by the within-task granularity analysis in Exp.~1.

\subsection{Relation to alternative objectives}\label{app:alternatives}

\paragraph{Information gain.}
For an edit $\sigma$ producing observation $o$ about an unknown model variable $M$, information gain is
\[
IG(\sigma)=H(M)-\mathbb E_{o\sim p(o\mid\sigma)}[H(M\mid o)].
\]
It rewards uncertainty reduction whether or not the resolved distinction affects achievable goals. Thus two tests can have equal information gain while only one changes $\Gamma_e$; grounded $\repemp$ distinguishes them through their projected outcomes.

\paragraph{Environmental empowerment.}
Environmental empowerment measures the capacity of the action-to-state channel, $\sup_\rho I_\rho(A_{1:T};S_T\mid S_0=s)$, with the agent's representation fixed. It can therefore be large when many physical states are controllably reachable but the model lacks a relation needed for a particular goal. Representational empowerment instead makes edits the channel inputs, and its grounded form retains only their consequences for $\Gamma_e$. The objectives accordingly govern different decisions: $\envemp$ selects how to act within a representation, whereas $\repemp$ evaluates changes to the representation itself.

\section{Causal Learning Environment}\label{app:model-selection-env}

Exp.~1 and Exp.~2 use the same causal-learning environment with different graphs, curricula, and observation-noise regimes. The environment is adapted from \citet{ke2021systematic} and extends standard causal discovery by allowing the learner to choose the granularity at which intervention effects are represented.

\subsection{Shared environment}\label{sec:env-shared}

\textbf{Variables and state.}
The environment instantiates a linear SEM over a set of latent variables $\mathcal{V}=\{V_1,\dots,V_n\}$; the size and connectivity of $\mathcal{V}$ are experiment-specific (\cref{sec:env-human,sec:env-sim}). Each variable carries a \emph{fundamental} scalar value $x\in\{1,2,\dots,16\}$ and three property \emph{roles} ($\texttt{strong}$, $\texttt{weak}$, and $\texttt{noise}$). Only $\texttt{strong}$ and $\texttt{weak}$ have causal mechanisms, while $\texttt{noise}$ is a distractor role so that participants and agents must discriminate causally relevant roles from irrelevant ones.
In Exp.~1, the mapping from surface labels (e.g., ``color,'' ``temperature'') to property roles is randomized across participants and held fixed within a participant. This prevents surface-label familiarity---such as a prior on semantic association---from biasing which properties are probed first or trusted more. 
In Exp.~2, role identity is fixed across seeds since the simulated agent has no priors to control for.

\textbf{Representation granularities.}
Each fundamental value $x$ is observed through one of three discretization maps:
    \[
    \phi_{\mathrm{fine}}(x)=\lceil x/2 \rceil \in \{1,\dots,8\},\quad
    \phi_{\mathrm{medium}}(x)=\lceil x/4 \rceil \in \{1,\dots,4\},\quad
    \phi_{\mathrm{coarse}}(x)=
    \begin{cases}
    \text{Low},& x\le 8,\\
    \text{High},& x> 8.
    \end{cases}
    \]
The granularity $\ell\in\Lambda=\{\texttt{coarse},\texttt{medium},\texttt{fine}\}$ is chosen by the learner on every step and determines only how effects are reported; the underlying propagation in the SEM always uses fundamental values.

\textbf{Mechanisms.}
All edges in the SEM are linear with additive effect noise:
\[
x_{\mathrm{child}} \;=\; e_{\mathrm{child}} + w\, x_{\mathrm{parent}} + \epsilon, \qquad \epsilon\sim\mathcal{N}(0,\sigma^2).
\]
The exogenous component $e_{\mathrm{child}}$ is back-solved from the task's initial state so that pre-intervention observations are consistent with the SEM. 

\textbf{Action space and task model.}
On every step the learner selects an observation granularity $\ell\in\Lambda$, then, conditional on $\ell$, intervenes on a variable-property pair $(v,p)\in\mathcal{V}_T\times\mathcal{P}$, where $\mathcal{V}_T\subseteq\mathcal{V}$ are the variables task $T$ makes visible and $v.p$ denotes property $p$ of variable $v$ (\cref{tab:gt-graph,tab:gt-curriculum}). The intervention sets $v.p$ to a new fundamental value, the mechanism above propagates it to every variable downstream of $v.p$, and the learner observes each affected variable through $\phi_\ell$. How the learner specifies that value is the one part of the action space the two experiments do not share (\cref{sec:env-human,sec:env-sim}). Granularity is the representational element, so the candidate space here is $\Cand=\Lambda$, whereas the intervention acts on the environment and leaves the representation unchanged, and \cref{app:human-action} scores the two decisions under separate rules. The task model $M_t$ is the learner's belief graph after $t$ steps, a directed graph over visible variable-property pairs with confidence-weighted edges, reconstructed from each participant's observations in Exp.~1 (\cref{app:human-methods}) and maintained as a Bayesian posterior in Exp.~2 (\cref{app:simulation-protocol}).

\begin{table}[t]
    \centering
    \scriptsize
    \caption{Ground-truth SEM for Exp.~1. Eight latent variables, nine directed edges with weights and signs. Edges 1-8 appear in at least one task window (\cref{tab:gt-curriculum}); edge 9 ($V_3 \to V_7$) is never jointly observable.}\label{tab:gt-graph}
    \begin{tabular}{clcc}
    \toprule
    \# & Edge & Weight & Sign \\
    \midrule
    1 & $V_1.\texttt{strong} \to V_2.\texttt{strong}$ & 0.70 & $+$\\
    2 & $V_1.\texttt{strong} \to V_3.\texttt{strong}$ & 0.60 & $+$\\
    3 & $V_2.\texttt{strong} \to V_4.\texttt{strong}$ & 0.70 & $+$\\
    4 & $V_3.\texttt{strong} \to V_4.\texttt{strong}$ & 0.50 & $+$\\
    5 & $V_4.\texttt{strong} \to V_5.\texttt{strong}$ & 0.70 & $+$\\
    6 & $V_5.\texttt{strong} \to V_6.\texttt{weak}$ & 0.85 & $+$\\
    7 & $V_5.\texttt{strong} \to V_7.\texttt{weak}$ & $-0.70$ & $-$\\
    8 & $V_6.\texttt{strong} \to V_8.\texttt{strong}$ & 0.85 & $+$\\
    9 & $V_3.\texttt{strong} \to V_7.\texttt{weak}$ & 0.50 & $+$\\
    \bottomrule
    \end{tabular}
\end{table}

\begin{table}[t]
    \centering
    \scriptsize
    \caption{Task curriculum for Exp.~1. Four tasks share variables across consecutive windows to enable cross-task transfer. Each task exposes a different structure (fork, merge, chain) and provides two planning goals reachable in $\le 5$ steps.}
    \begin{tabular}{llll}
    \toprule
    Task & Variables $\mathcal{V}_t$ & Visible edges & Structure \\
    \midrule
    $T_1$ & $\{V_1,V_2,V_3\}$       & $V_1\!\to\!V_2,\; V_1\!\to\!V_3$                         & Fork \\
    $T_2$ & $\{V_2,V_3,V_4,V_5\}$   & $V_2\!\to\!V_4,\; V_3\!\to\!V_4,\; V_4\!\to\!V_5$        & Merge + chain \\
    $T_3$ & $\{V_4,V_5,V_6,V_7\}$   & $V_4\!\to\!V_5,\; V_5\!\to\!V_6\,(+),\; V_5\!\to\!V_7\,(-)$ & Fork with negative edge \\
    $T_4$ & $\{V_5,V_6,V_7,V_8\}$   & $V_5\!\to\!V_6\,(+),\; V_5\!\to\!V_7\,(-),\; V_6\!\to\!V_8$ & Deep chain \\
    \bottomrule
    \end{tabular}
    \label{tab:gt-curriculum}
\end{table}

\subsection{Experiment~1 instantiation: human experiment}\label{sec:env-human}

The human experiment uses a fixed eight-variable and nine-edge causal graph. The full edge list, including mechanism class, weight, and sign, is given in \cref{tab:gt-graph}. The four-task curriculum with visible edges and goals is given in \cref{tab:gt-curriculum}; tasks share variables across the curriculum so that overlapping windows expose the same underlying SEM through different visible subsets.

Participants never see fundamental values, so they specify an intervention in the bins of the selected granularity. Each step moves the chosen property one bin up or down from its current value at $\ell$, and the environment sets the underlying fundamental value to the midpoint of the target bin.

\subsection{Experiment~2 instantiation: simulation}\label{sec:env-sim}

Exp.~2 replaces the fixed eight-variable causal graph with a ten-variable graph randomly sampled per seed. For each of $30$ seeds we draw a SEM over $n_{\mathrm{var}}{=}10$ variables with edge probability $p_e{=}.25$, instantiated under the same SEM mechanism (linear edges, additive Gaussian noise) and observed through the same $\phi_\ell$ maps as Exp.~1. This avoids any dependence of the findings of the granularity-intervention mechanisms on a single hand-designed graph.

Each seed's curriculum is generated procedurally: a sequence of $5$ training tasks followed by $5$ held-out test tasks, each task exposing an overlapping window of $4$--$7$ variables drawn from the seed's graph. Visible edges in a task are the induced subgraph on its visible variables. Consecutive tasks share variables and edges, supporting the same cross-task transfer analyses as the human curriculum. Full per-task generation parameters are in \cref{app:simulation}.

\section{Experiment~1 Supplementary Results}\label{app:human}

\subsection{Participants and procedure}\label{app:human-protocol}

We recruited $n{=}70$ adults through Prolific ($50\%$ female; age $18$-$40$, $M{=}30.68$, $SD{=}6.27$). Participants provided informed consent and were compensated above the local minimum wage, calibrated to the median completion time. The study was approved by the Institutional Review Board at the University of California, Berkeley. After instructions and a worked example, participants completed four tasks and a debrief. Each task comprised up to $20$ exploration steps, a causal-graph report, and two goal-directed tests. At each exploration step, participants selected a granularity $\ell$ and then a variable-property intervention $(v,r)$; they observed only the resulting representation at $\ell$.

We logged granularity choices, intervention choices, observed effects, response time, reported graphs, and test performance. The analyses use $3{,}299$ valid exploration decisions. Model comparisons retain steps on which all alternatives were admissible and every rule returned a finite score under the participant's current belief state.

\subsection{Behavioral analyses}\label{app:human-behavior}

\textbf{Choice persistence.}
Three features of granularity choice motivate the control model that the criterion comparison scores against (\cref{app:human-methods}). Participants repeat the preceding granularity on $67.2\%$ of consecutive within-task steps, compared with $46.4\%$ under a participant-level shuffle that preserves how often each participant uses each granularity but discards the order ($p{<}.001$). $k$-means on participant granularity shares identifies balanced, medium-preferring, and fine-preferring policies, and split-half membership agrees for $51.4\%$ of participants against a $39.2\%$ shuffled baseline ($z{=}2.09$, $p{=}.036$), so participants differ stably in which granularity they favor. A multinomial logit predicting the next granularity shows that the preceding granularity dominates the latest visible intervention effect, contributing $+619.8$ held-out log-likelihood units ($\chi^2{=}1239.7$, coefficient $+1.34$) against $+14.3$ ($\chi^2{=}28.6$); \cref{fig:level_selection_prediction} ranks every predictor in that model. Granularity choice is therefore persistent and largely insensitive to the most recent observation, even though the aggregate distribution still shifts within and across tasks, so the criterion comparison credits a rule only with what it adds beyond stickiness and participant-level preference.

\begin{figure}[t]
    \centering
    \includegraphics[width=0.8\textwidth]{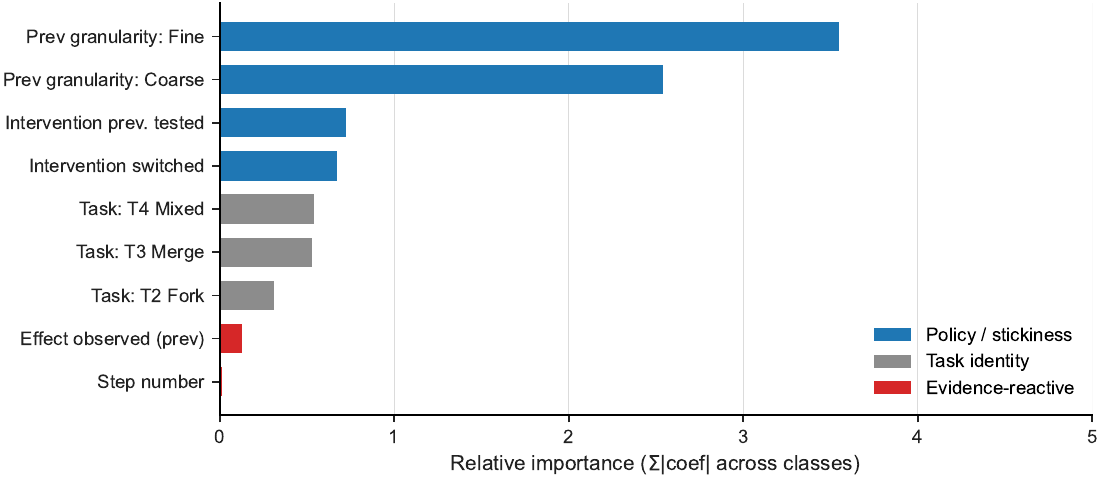}
    \caption{\textbf{Predictors of granularity choice.} Relative importance of each predictor in a multinomial logit over the three granularities, measured as the summed absolute coefficient across the three classes. The preceding granularity outweighs every other predictor by roughly fourfold, task identity contributes moderately, and the evidence available from the previous step---whether that step produced a visible effect, and how far into the task it occurred---contributes least.}
    \label{fig:level_selection_prediction}
\end{figure}

\begin{figure}[t]
    \centering
    \includegraphics[width=\textwidth]{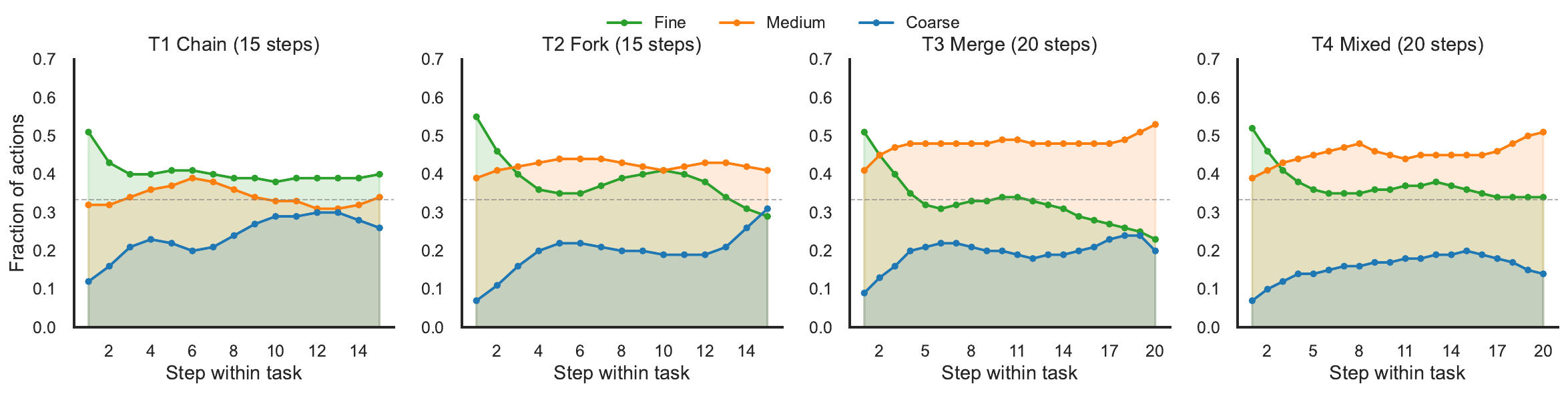}
    \caption{\textbf{Granularity use by task and step.} Across the four tasks, \texttt{medium} use increases while \texttt{coarse} use decreases. Within a task, \texttt{fine} use decreases and \texttt{coarse} use increases from the first to the second half of exploration.}
    \label{fig:level_choice_env_and_step_wise}
\end{figure}

\textbf{Compression within and across tasks.}
Granularity use compresses within a task but converges across tasks, and the two patterns move \texttt{coarse} in opposite directions. Within a task, participants who return to a source they have already probed view it at a coarser granularity more often than a finer one ($557$ vs.~$430$, two-sided binomial test, $p{<}.001$); splitting each task at its median step, \texttt{coarse} use rises from $17.7\%$ to $22.3\%$ and \texttt{fine} use falls from $38.3\%$ to $34.5\%$, while \texttt{medium} use is flat ($44.0\%$ to $43.2\%$). Across tasks the direction of the \texttt{coarse} trend reverses: comparing each participant's first and last task, \texttt{medium} use rises from $35.0\%$ to $45.9\%$ ($t(69){=}2.83$, $p{=}.006$) and \texttt{coarse} use falls from $24.7\%$ to $16.6\%$ ($t(69){=}{-}2.86$, $p{=}.006$), while \texttt{fine} use does not change reliably ($40.3\%$ to $37.5\%$, $t(69){=}{-}0.67$, $p{=}.50$).
This is the stage-dependence predicted by \cref{app:prop:nonstationary}. Detail acquired during discovery is dropped once the local structure is known, whereas across tasks participants settle on the cheapest representation that remains goal-sufficient which \texttt{coarse} is not. Taskwise trajectories are in \cref{fig:level_choice_env_and_step_wise}.

\textbf{Granularity use and goal performance.}
The correlations reported in \cref{sec:human} are between-participant associations across the $70$ participants; \texttt{coarse} use predicts neither measure ($r{=}.10$, $p{=}.43$ for achievement rate; $r{=}{-}.14$, $p{=}.27$ for steps per goal). To distinguish participants' overall \texttt{medium} use (the between-participant component) from each task's deviation from that participant's average (the within-participant component), we fit a Mundlak logistic mixed model over the $560$ binary goal outcomes, with task fixed effects and a participant random intercept. Overall \texttt{medium} use was positively associated with goal achievement ($\beta{=}1.29$, $\mathrm{SE}{=}.62$, $z{=}2.07$, $p{=}.039$), whereas the task-specific deviation was positive but unreliable ($\beta{=}.59$, $\mathrm{SE}{=}.58$, $z{=}1.01$, $p{=}.31$).

\textbf{Effort and response time.}
Across $3{,}270$ valid deliberation intervals, median RT is nearly identical for \texttt{coarse}, \texttt{medium}, and \texttt{fine} choices ($7.6$/$7.6$/$7.3$\,s). The \texttt{fine}-\texttt{medium} contrast on log-RT, computed within participants by mean-centering each participant's log-RTs before comparison, shows no reliable difference ($\beta{=}-0.003$, $p{=}.94$). Entered into the criterion comparison of \cref{app:human-methods}, a participant-residualized RT score adds no held-out fit ($-0.16$ $[-0.66,+0.39]$) and leaves the $\repemp$ increment essentially unchanged when the two are fit jointly ($+12.16$). These tests do not support effort as the source of the granularity pattern.

\subsection{Criterion comparison}\label{app:human-methods}

\textbf{Belief-graph reconstruction.}
For each participant, we reconstruct the belief graph $M_t$ online from that participant's observations. The prespecified update uses detection threshold $\theta{=}0$, one confirmation per edge, and attributes all visible effects to the intervened variable. No scoring rule observes the ground-truth graph. We vary these assumptions in the specification curve below.

\textbf{Scoring the granularity decision.}\label{app:human-action}
We score the granularity and intervention decisions separately because they have different action spaces. For granularity, $\repemp$ is the grounded channel of \cref{eq:goal-proj-repemp} on the fixed probe-goal set,
\[
  \repemp_\ell(\ell)=I\!\left(U;\Gamma_{\Gach}(\ell)\mid M_t\right).
\]
A candidate intervention $U$ achieves a goal when its predicted outcome under $M_t$ lands within the goal's bin width; only goals expressible at $\ell$ contribute. Thus, additional resolution cannot increase the score beyond the task's goal set. $\envemp_\ell$ instead counts distinguishable commandable states without projecting onto goals,
\[
  \envemp_\ell(\ell)=\log_2\!\sum_u B(\ell)^{|\mathrm{reach}(u;M_t)|},
\]
where $B(\ell)\in\{2,4,8\}$. The InfoGain rule scores the probability of detecting unresolved structure times its remaining posterior entropy.

\textbf{Scoring the intervention decision.}
$\envemp_{(v,p)}$ scores expected downstream reach from $(v,p)$, including prior-weighted reach through untested edges. $\repemp_{(v,p)}$ scores the expected increase in plannable goals if the intervention is tested. InfoGain sums the entropy of unknown outgoing edges.

\textbf{Control and fitting.}
The hierarchical control contains group-level choice preferences, shrunken participant deviations, previous-choice stickiness, and, for intervention choice, intervention familiarity. It gains $+928$ nats over uniform for granularity and $+475$ for intervention choice. Candidate utilities are therefore evaluated only by the incremental fit they add to this control. Each utility enters a regularized softmax ($\ell_2{=}2$) with one population coefficient and is fit by leave-one-task-out cross-validation. We report held-out $\Delta$LL beyond this control (\cref{fig:human_results}c).

\textbf{Specification curve.}\label{app:human-spec-curve}
Across $12$ belief-update specifications (detection threshold $\times$ edge confirmations $\times$ effect attribution), fixed-goal $\repemp_\ell$ is the strongest granularity rule in a plurality of cells, but its median increment is near zero and no positive interval persists across the curve. Rules whose goal set grows with resolution are uniformly negative. The $\envemp$ intervention point estimate is stable to a $16\times$ range of control shrinkage and a time-in-task control, but remains inferentially inconclusive.

\section{Experiment~2 Supplementary Results}\label{app:simulation}

\subsection{Simulation protocol}\label{app:simulation-protocol}

The graph and curriculum for each seed are specified in \cref{sec:env-sim}. Each task permits $1{,}000$ interventions ($10$ episodes of $100$ steps) and the agent carries its beliefs across task boundaries. The agent uses the same Bayesian belief tracker in every condition; an edge is counted as learned when its posterior exceeds $.5$. Each condition pairs a granularity rule with an intervention rule, drawn from $\repemp$, $\envemp$, InfoGain, and uniform random; the Factored condition pairs $\repemp$ granularity with $\envemp$ intervention selection. Observation noise is parameterized as $(\sigma_{\texttt{coarse}},\sigma_{\texttt{medium}},\sigma_{\texttt{fine}})$. We evaluate flat $(0.5,0.5,0.5)$, R-\texttt{coarse} $(0.5,2.0,3.0)$, R-\texttt{medium} $(2.0,0.5,3.0)$, and R-\texttt{fine} $(2.0,2.0,0.5)$. R \texttt{medium} is the default regime in the main comparison.

At each logged step, we compare the belief graph with the ground-truth graph using F1 and normalized SHD. The ground-truth edge set is projected through the condition's discretization $\phi_\ell$; we never up-project a learned graph to unsupported detail. Conditions, seeds, and tasks are fully crossed, yielding $150$ paired seed--task instances per split. Contrasts use paired $t$-tests with Bonferroni correction ($\alpha{=}.05/9$); uniform random is the reference floor.

\subsection{Granularity selection across noise regimes}\label{app:sim-identifiability}

We validate each regime without consulting an adaptive selection rule. Structural identifiability measures the fraction of true edges for which an intervention produces an observational distribution distinguishable from the no-edge distribution at the available per-intervention budget. As a behavioral cross-check, fixed-granularity agents select interventions uniformly. The two criteria agree on the graph-recovery optimum in each regime (\cref{tab:app-sim-identifiability}). This validation concerns graph recovery; it does not guarantee agreement with $\repemp$, which scores commandable goal outcomes.

\begin{table}[t]
\centering
\scriptsize
\caption{Regime validation. Bold entries identify the best granularity under structural identifiability and fixed-policy graph recovery.}
\label{tab:app-sim-identifiability}
\begin{tabular}{lccc c ccc}
\toprule
& \multicolumn{3}{c}{Identifiable edges (\%)} & & \multicolumn{3}{c}{Fixed-granularity F1} \\
\cmidrule(lr){2-4}\cmidrule(lr){6-8}
Regime & \texttt{coarse} & \texttt{medium} & \texttt{fine} & & \texttt{coarse} & \texttt{medium} & \texttt{fine} \\
\midrule
flat & 96.0 & 99.7 & \textbf{100.0} & & .520 & .708 & \textbf{.753} \\
R-\texttt{coarse} & \textbf{96.0} & 67.2 & 19.5 & & \textbf{.520} & .370 & .185 \\
R-\texttt{medium} & 49.4 & \textbf{99.7} & 19.5 & & .343 & \textbf{.708} & .185 \\
R-\texttt{fine} & 49.4 & 67.2 & \textbf{100.0} & & .343 & .370 & \textbf{.753} \\
\bottomrule
\end{tabular}
\end{table}

\begin{table}[t]
\centering
\scriptsize
\caption{Share of \texttt{coarse}/\texttt{medium}/\texttt{fine} decisions by rule and regime; the modal level is bold. The final row gives the independently validated graph-recovery optimum.}
\label{tab:app-sim-tracking}
\begin{tabular}{lcccc}
\toprule
Rule & flat & R-\texttt{coarse} & R-\texttt{medium} & R-\texttt{fine} \\
\midrule
$\repemp$ & .02/.02/\textbf{.97} & .02/.08/\textbf{.91} & .02/\textbf{.63}/.35 & .02/.02/\textbf{.97} \\
$\envemp$ & .02/.02/\textbf{.97} & .02/.02/\textbf{.97} & .02/.02/\textbf{.97} & .02/.02/\textbf{.97} \\
InfoGain & .34/.33/.33 & .34/.33/.33 & .34/.33/.33 & .33/.33/.33 \\
\midrule
Graph-recovery optimum & \texttt{fine} & \texttt{coarse} & \texttt{medium} & \texttt{fine} \\
\bottomrule
\end{tabular}
\end{table}

\Cref{tab:app-sim-tracking} reports the share of decisions each rule spends at each granularity. The diagnostic positive result is the R-\texttt{medium} shift, where $\repemp$ moves from $97\%$ \texttt{fine} in the flat regime to $63\%$ \texttt{medium}, then returns to $97\%$ \texttt{fine} in R-\texttt{fine}. $\envemp$ remains fine-preferring, and InfoGain remains approximately uniform.

The shipped R-\texttt{coarse} regime is a negative result. Although fixed-policy graph recovery is highest at \texttt{coarse}, $\repemp$ selects \texttt{fine} on $91\%$ of decisions and obtains $.139$ held-out F1, the divergence between graph recovery and commandable goal outcomes noted above. Its analytic preference for \texttt{coarse} requires $\sigma_{\texttt{medium}}>3.52$ when $\sigma_{\texttt{coarse}}{=}0.5$; the shipped value is $2.0$.

Post hoc sweeps locate, but do not remove, this boundary. Raising $\sigma_{\texttt{medium}}$ to $4.02$ increases the full-budget coarse share only to $23\%$. At the same noise and a budget of $20$, the coarse share reaches $59\%$ (modal in $20/30$ seeds), but F1 falls to $.176$, versus $.326$ at the full budget. We treat these sweeps as exploratory diagnostics, not evidence that $\repemp$ robustly tracks coarse-optimal graph-recovery regimes.

\section{Grid World Environments and Curator-Actor Architecture}\label{app:model-construction}

\begin{algorithm}[t]
\caption{The Curator-Actor loop for continual model construction}
\label{alg:cmc}
\small
\begin{algorithmic}[1]
\STATE \textbf{Input:} environment $e$, library $\Lib$, proposal
vocabulary $\Cand$, budgets $B_{\mathrm{task}}, B_{\mathrm{lib}}$
\STATE $\mathcal{D}\leftarrow\varnothing$;\quad
$M_0\leftarrow\mathcal{C}.\textsc{Seed}(\Lib,e)$
\hfill {\footnotesize\emph{// initial task model from $\Lib$}}
\FOR{$t=1,\ldots,B_{\mathrm{task}}$}
    \STATE $(M_t,\mathcal{G}_t,\{\pi_g\}_{g\in\mathcal{G}_t})\leftarrow
    \mathcal{C}.\textsc{Propose}(\Lib,o_e,\mathcal{D};\,M_{t-1},\Cand)$
    \hfill {\footnotesize\emph{// task model, probe goals, symbolic plans}}
    \FOR{each $g\in\mathcal{G}_t$}
        \STATE $(\tau_g,\Gamma_e(M_t)_g)\leftarrow
        \mathcal{A}.\textsc{Execute}(M_t,g,\pi_g)$
        \hfill {\footnotesize\emph{// primitive grounding only}}
    \ENDFOR
    \STATE $\widehat V_t(\Delta M_t)\leftarrow
    \mathcal{C}.\textsc{Score}\bigl(\Gamma_e(M_t),\Gamma_e(M_{t-1});
    \mathcal{G}_t\bigr)$,
    where $\Delta M_t = M_t\setminus M_{t-1}$
    \hfill {\footnotesize\emph{// \cref{eq:expected-action-value}, summed over $\sigma\in\Delta M_t$}}
    \STATE \textbf{if} $\widehat V_t(\Delta M_t)>0$ \textbf{then} accept $M_t$,
    append $\{\tau_g\}$ to $\mathcal{D}$;
    \textbf{else} $M_t\leftarrow M_{t-1}$ and reject the candidate
    elements (quarantine)
    \IF{$\Gamma_e(M_t)$ has plateaued on the probe set}
        \STATE Backward-prune any element of $M_t$ whose removal
        leaves $\Gamma_e(M_t)$ unchanged
    \ENDIF
\ENDFOR
\STATE Re-evaluate each retained element's $\widehat V_t$ on the
cross-environment probe buffer (regime-appropriate estimator)
\STATE $\Lib\leftarrow
\mathcal{C}.\textsc{GreedyCurate}\bigl(\Lib\cup M_t,\,
B_{\mathrm{lib}}\bigr)$
\hfill {\footnotesize\emph{// greedy marginal-coverage update}}
\end{algorithmic}
\end{algorithm}

This section documents the open-vocabulary experiments referenced in \cref{sec:exp-construction}: the Curator-Actor architecture, the two environment families (BabyAI and Zelda), and the training and evaluation protocols. The implementation accompanying these experiments will be released with the paper.

\subsection{Curator-Actor architecture}\label{app:curator-system}

\textbf{Library and task model.}
The persistent library $\Lib$ is a typed PDDL knowledge store with types, type-hierarchy edges, predicates, and operators. Types are atomic symbols (e.g., \texttt{agent}, \texttt{key}, \texttt{door}, \texttt{object}) and type-hierarchy edges record subtype relations (\texttt{key - object}, \texttt{door - object}). Predicates carry a PDDL signature and a Python grounding closure that evaluates the predicate on raw observations:

{
\small
\begin{verbatim}
(:predicates
  (holding ?a - agent ?o - object)
  (next_to ?a - agent ?o - object)
  (locked  ?d - door)
  (is_key_for ?k - key ?d - door))
\end{verbatim}
}

Operators are PDDL action schemas with typed preconditions and effects:

{
\small
\begin{verbatim}
(:action unlock_door
  :parameters (?a - agent ?k - key ?d - door)
  :precondition (and (holding ?a ?k)
                     (next_to ?a ?d)
                     (locked  ?d)
                     (is_key_for ?k ?d))
  :effect       (and (not (locked ?d))))
\end{verbatim}
}

Every element carries a status flag in $\{\texttt{candidate},\texttt{trusted}\}$. The task-local model $M_t$ is the PDDL domain assembled from the \texttt{trusted} subset of $\Lib$ relevant to the current environment, plus any \texttt{candidate} elements currently under evaluation, plus the grounding code needed to evaluate predicates on raw observations.

\textbf{Curator.}\label{app:curator-scoring}
The Curator~$\mathcal{C}$ is the LLM-backed component. It runs the following loop:
\begin{enumerate}[leftmargin=1.5em,itemsep=0.1em]
    \item \emph{Propose.} From the experience buffer $\mathcal{D}_t$, the current $M_t$, and $\Lib$, an LLM proposer emits candidate types, predicates (with Python grounding code), and operator schemas. The proposer also revises existing elements (refine signature, split, compose) and emits top-environment goal predicates that are expressible under $M_t\cup\Cand_t$ but not currently achievable under $M_t$. 
    \item \emph{Probe.}  For each candidate $a$, $\mathcal{C}$ constructs a probe goal set $\mathcal{G}_t(a)$: goals blocked under $M_t$ but expressible under $M_t\!\oplus\!a$ (positive probes) and goals already solvable that act as regression checks (negative probes). For each $g \in \mathcal{G}_t(a)$, $\mathcal{C}$ compiles $g$ into a PDDL problem against $M_t$ and calls Fast Downward~\citep{helmert2006fast}. When the PDDL compile fails (e.g., a candidate operator has malformed typing) or the external solver times out, $\mathcal{C}$ falls back to a fuzzy beam-search planner that operates directly on the operator schemas in $M_t$. In both cases the output is an operator sequence $\pi_g^{M_t} = (a_1,\dots,a_k)$, which is handed to the Actor along with $g$.
    \item \emph{Score.} Candidates $a$ are evaluated with \cref{eq:expected-action-value}. The Actor estimates the change in achievement probability on the witness set $W_t(a)$; candidate validity weights this gain, and the representational penalty determines whether the marginal benefit is large enough for promotion. To save Actor calls, candidates are first pre-ranked by a cheap surrogate computed from $\Lib$, $\mathcal{D}_t$, and the candidate's typed signature. Only top-ranked candidates are executed; the resulting signatures $\Gamma_e(M_t\!\oplus\!a)_g$ and $\Gamma_e(M_t)_g$ determine the promotion score.
    \item \emph{Curate.}  Candidates that meet minimum thresholds on $\widehat V_t$, on reliability (success rate across probe attempts), and on a minimum test count are promoted to \texttt{trusted}.  At the cross-environment refactor stage described in \cref{sec:method}, the same $\widehat V_t$ is re-evaluated against the cross-environment probe buffer---which samples a fixed budget of imagined state-goal pairs from the buffer and scores reachability under $\Lib$ vs.\ $\Lib\cup\{\sigma\}$---and the trusted set is greedily re-ranked under the library budget by marginal $\widehat V_t$.
\end{enumerate}

Between curriculum environments, the final library $\Lib$ from the previous environment is carried forward unchanged and seeds $M_t$ on the next environment.

\textbf{Actor.}\label{app:actor-local-wm}
The Actor receives a $(M_t, g, \pi_g^{M_t})$ triple from the Curator and returns evidence about $\Gamma_e(M_t)_g$. No LLM is involved at plan-and-execute time, and the Actor never plans symbolically: its job is to realize the operator sequence $\pi_g^{M_t} = (a_1,\dots,a_k)$ in primitive actions and report whether each $a_i$ landed its effects.
Each operator $a_i$ is realized as a primitive sequence by heuristic search over the agent's primitive action set.  The search uses an agent-egocentric local world model: a small representation $(\mathrm{patch},\mathrm{carried\_slot})$ of the agent's neighborhood, where $\mathrm{patch}$ is a fixed-radius window of cells rotated so the agent heading is up, and $\mathrm{carried\_slot}$ encodes the carried object outside the cell grid.  Predicates that bind to in-window objects, plus agent-indexed predicates (e.g.\ \texttt{holding}, \texttt{facing}), are projected into the local model on demand by running their grounding closures.
The search heuristic counts unmet operator preconditions under the projected predicates, a goal-count relaxation in the spirit of classical planning heuristics; we make no admissibility claim, and the environment verifies the resulting predicate flips. The local model is learned online: a self-supervised predictor maps $(\mathrm{patch},\mathrm{carried\_slot},\mathrm{primitive})$ to the local predicate-flip set and is updated from the Actor's primitive trajectories. Calls into the live environment during search are amortized by rollouts in this learned model.

\subsection{Environments}

\textbf{BabyAI.}
We use four BabyAI environments for training: \texttt{GoToLocal}, \texttt{Pickup}, \texttt{OpenRedDoor}, and \texttt{UnlockLocal}. All four share the same primitive action space (turn, move, pickup, drop, toggle). The ordering is intended to give the Curator opportunities to compose elements promoted at earlier environments (e.g., navigation routines reused for key acquisition before unlocking).
A held-out split of four environments, never seen during training, tests transfer to configurations whose goal templates differ from the training environments: \texttt{PutNextLocal} introduces spatial object--object relations; \texttt{MoveTwoAcrossS5N2} requires multi-object manipulation; \texttt{PickupAbove} changes the relative orientation of agent and target; and \texttt{Unlock} is a variant of \texttt{UnlockLocal} with different layout and object constraints. \cref{fig:babyai-envs} shows representative screenshots.

\begin{figure}[t]
\centering
    \includegraphics[width=0.8\textwidth]{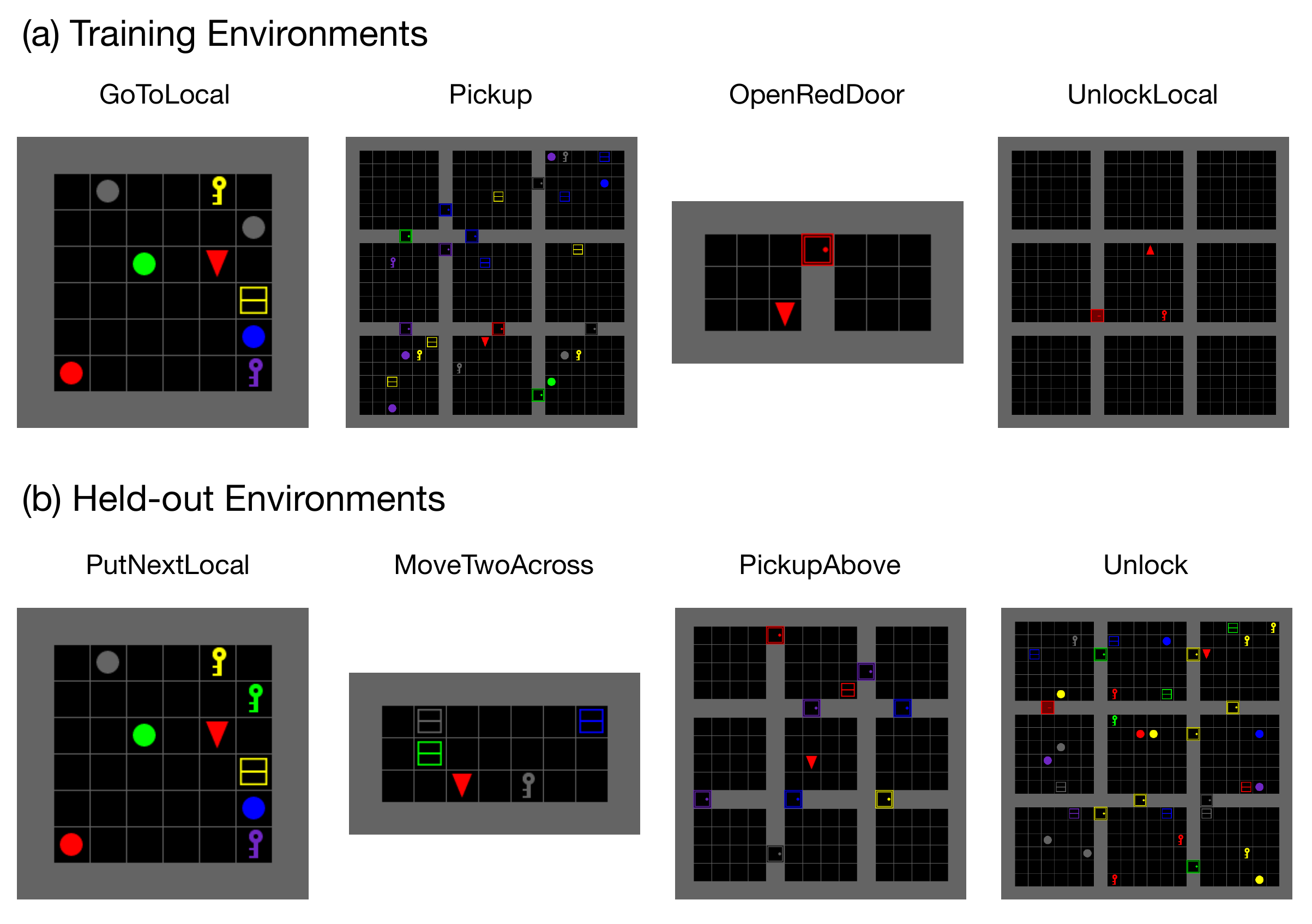}
    \caption{\textbf{BabyAI training and held-out environments used in Exp.~3.} Training environments are ordered by minimum operator-chain length; held-out environments combine those operators in configurations not seen during training.}
\label{fig:babyai-envs}
\end{figure}

\textbf{Zelda (EMPA).}
We use the \emph{Zelda} game from the EMPA suite~\citep{tsividis2021human}, which provides typed objects---player, monster, key, door, wall---and a native reward function. The first three environments are used for training and the last two are held out for evaluation (\cref{fig:zelda-envs}). Held-out environments share the primitive dynamics of training but introduce new layouts and precondition configurations, testing whether predicates and operators promoted during training (e.g., key acquisition, monster avoidance) transfer to settings where the same goals must be reached through different spatial arrangements.

\begin{figure}[h]
\centering
    \includegraphics[width=0.8\textwidth]{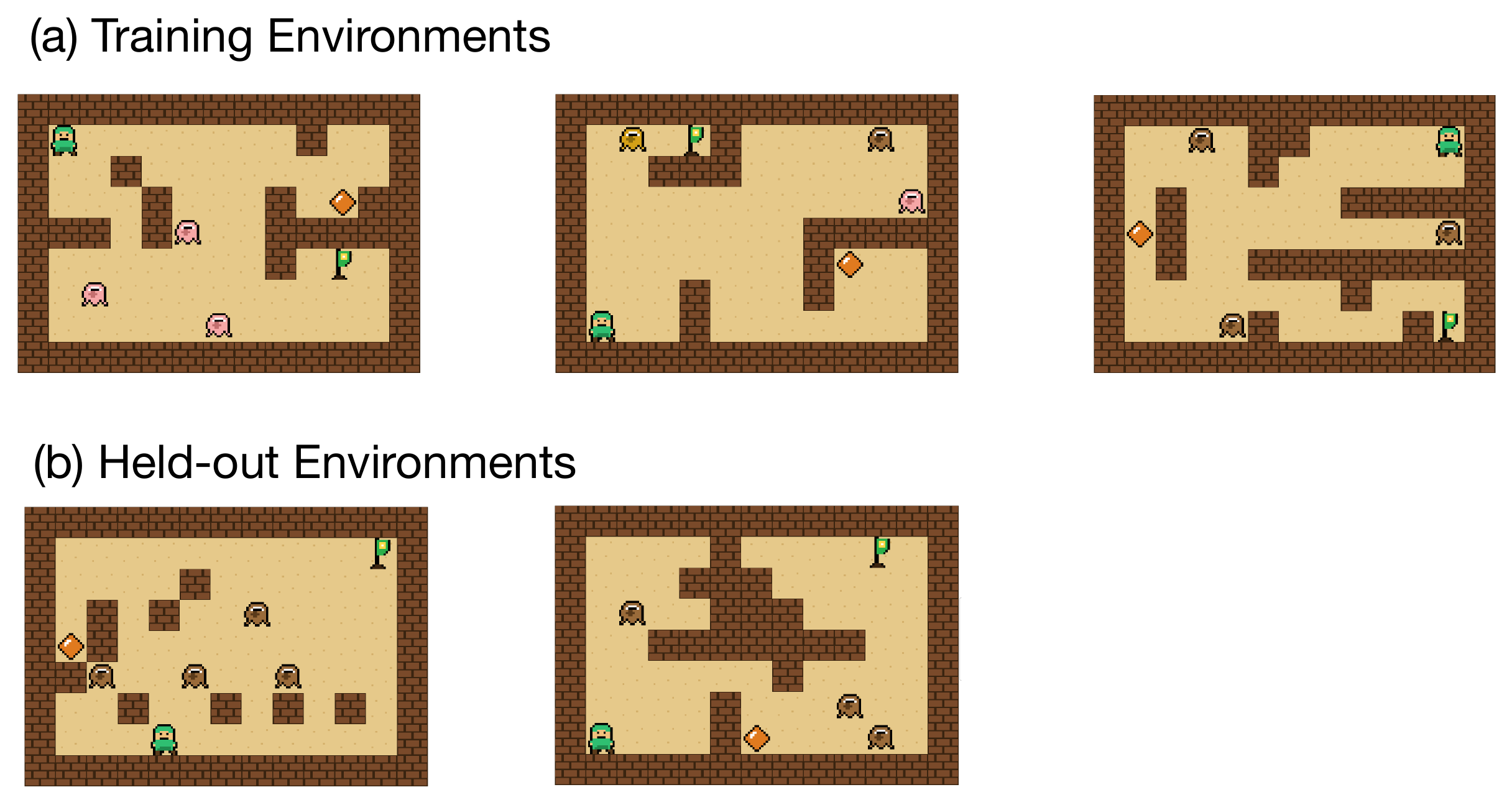}
    \caption{\textbf{Zelda training and held-out environments used in Exp.~3.}  The first three environments are used for training and the last two for held-out evaluation; primitive dynamics are identical across environments but layouts and precondition configurations differ.}
    \label{fig:zelda-envs}
\end{figure}

\subsection{Training and evaluation}\label{app:curator-protocol}

\textbf{Training budget.}
Within an environment, the Curator-Actor loop runs until either the candidate set stabilizes---no new promotions for $K_{\mathrm{stab}}{=}3$ consecutive cycles---or the per-environment LLM-call budget is exhausted. We allocate $200$ LLM calls per environment on BabyAI and $400$ on Zelda, matching WorldCoder's budget on each family for direct comparison. Motif's pipeline issues an LLM call for every labeled trajectory pair and consequently uses $\sim 5{,}000$ calls per environment; this is a structural cost of preference labeling, not a compute advantage. PPO has no LLM calls and is matched to Motif on environment-step budget instead.

Crucially, the LLM-call budget translates into a small environment-step footprint. Curator and WorldCoder each consume on the order of thousands of primitive environment steps per environment---two orders of magnitude below the $\sim 500\mathrm{k}$ and $\sim 1\mathrm{m}$ steps used by Motif and by PPO---because each LLM call drives at most a handful of probe-goal executions rather than continuous policy rollout. We treat this gap as a property of symbolic methods, but flag it explicitly so the transfer comparisons are not read as unequal-compute results.
Immediately after an environment finishes training, we record a \emph{self-evaluation} score on the same environment using the then-current $\Lib$; this score anchors the backward-transfer comparison defined below.

All LLM-backed components (Curator's proposer, scorer, and curator; WorldCoder's synthesizer; Motif's labeler) use the same model: \texttt{deepseek-r1:70b} served locally via Ollama on a single A100 GPU.
We report mean and standard deviation across $3$ seeds per condition, with seeds controlling LLM sampling and environment stochasticity. Within a seed, evaluation is averaged over $N_{\mathrm{eval}}=50$ rollouts per (environment, stage) cell on BabyAI and Zelda. The success metric, $\mathrm{reward\_success}$, is the fraction of evaluation rollouts in which the agent reaches the environment's primary goal as defined by the environment's native reward function. A rollout terminates on goal achievement, agent death (Zelda), or step-limit ($T_{\mathrm{max}} =256$ primitive steps).

\textbf{Promotion criteria.}
A candidate element $\sigma$ is promoted to \texttt{trusted} when it has reliability at least $.70$, has been dispatched on at least three probes, and its unpenalized marginal coverage gain exceeds $.50$. With $c_\sigma\equiv1$ and $\lambda=.50$, the final condition is equivalent to a positive penalized score in \cref{eq:expected-action-value}.
At the environment boundary, $\widehat V_t$ is re-evaluated against a cross-environment probe buffer using the \texttt{imagined\_goal} estimator. We sample $N_{\mathrm{imag}}=50$ state-goal pairs uniformly from the buffer, plan against $\Lib$ and $\Lib \cup \{\sigma\}$ for each, and set $\sigma$'s contribution equal to the number of newly reachable goals. 

\textbf{Forward transfer.}
For each adjacent pair $(\mathrm{env}_i, \mathrm{env}_{i+1})$ in the training curriculum, we freeze $\Lib$ at the end of $\mathrm{env}_i$ and evaluate it on $\mathrm{env}_{i+1}$ \emph{without further library updates}: the agent receives no training trajectories on $\mathrm{env}_{i+1}$ and the Curator does not propose new candidates. The local world model (\cref{app:curator-system}) continues to update from primitive trajectories, since this update happens at no extra environment budget. This isolates the short-range generalization of promoted symbols.

\textbf{Backward transfer.}
After the full curriculum has been trained, we re-evaluate the final $\Lib$ on each earlier training environment under the same frozen-library protocol as forward transfer. Backward transfer is the difference between an environment's success rate immediately after it was trained and its success rate after all subsequent training has finished:
\[
\mathrm{Backward}_i \;=\; \text{final-eval}(i) - \text{self-eval}(i),
\]
where $\text{self-eval}(i)$ is the score recorded immediately after $\mathrm{env}_i$ was trained and $\text{final-eval}(i)$ is the score under the final $\Lib$. Negative values indicate forgetting; positive values indicate retroactive improvement, where abstractions promoted at later environments make earlier environments easier. Rather than collapse this to a scalar, \cref{fig:curator-backward-babyai,fig:curator-backward-zelda} plot the full success trajectory $\text{eval}(i, j)$ at every evaluation stage $j \ge i$, since the shape of the trajectory distinguishes monotonic accumulation (Curator) from peak-and-decay (WorldCoder, PPO).

\textbf{Held-out transfer.}
For each environment family we hold out a disjoint set of environments not used during training. The frozen final $\Lib$ is evaluated zero-shot on each held-out environment under the same protocol as forward and backward transfer: no training trajectories, no candidate proposals, only library lookup and Actor execution. This tests compositional transfer to goal templates and configurations whose specific combinations did not appear in training.

\subsection{Prompt templates}\label{app:prompts}

The Curator-Actor pipeline issues two LLM calls per cycle to construct $M_t$, plus one call to propose a probe problem. 
Each Curator cycle first calls \emph{Discover} to propose new actions, predicates, and fixes from recent trajectories, then calls \emph{Formalize} to turn those proposals into typed PDDL schemas and Python grounding code. The split mirrors the conceptual distinction between \emph{what to propose} (a generative step that draws on both observed trajectories and prior knowledge) and \emph{how to express it formally} (a compilation step that enforces syntactic and type constraints).
We show the domain-agnostic templates below; \texttt{\{\{...\}\}} marks runtime substitutions, and the full prompts (with environment-specific scaffolding for BabyAI position tuples, Zelda direction conventions, etc.) are released with the code.

\begin{promptbox}[Construct stage 1 --- Discover]
\textbf{System.}

You are an expert in symbolic abstraction and world-model learning. Operate at a high level of abstraction: read low-level observations but lift them into abstract symbolic descriptions suitable for task-level planning. You may draw on both observed evidence and prior knowledge of the domain; both are valid.

\medskip
\textbf{User.}

Analyze the interaction experiences and propose high-level symbolic structure for task-level planning.

\textsc{Existing domain}

\textsc{Experiences:}\quad \promptvar{trajectories}

\textsc{Environment metadata:}\quad \promptvar{env\_metadata}

\medskip
\textsc{Task.}
\begin{enumerate}
\item Identify high-level patterns in state transitions.
\item Propose new operators, predicates and types.
\item If an existing symbol almost fits, propose a fix rather than a new symbol.
\end{enumerate}

\medskip
\textsc{Output} (strict JSON):
\begin{verbatim}
{
  "observed_patterns": [{"description": "...", "confidence": 0..1}],
  "proposed_actions": [
    {"name": "...", "parameters": [...],
     "preconditions_sketch": "...", "effects_sketch": "...",
     "evidence": {"type": "experience-driven"|"gap-driven", ...},
     "justification": "..."}
  ],
  "proposed_predicates": [{"name": "...", "meaning": "...", ...}],
  "fix_proposals":      [{"name": "...", "current_issue": "...",
                          "suggested_fix": "..."}]
}
\end{verbatim}
\end{promptbox}

\begin{promptbox}[Construct stage 2 --- Formalize]
\textbf{System.}

You are a compiler, not a generator: stay faithful to what the patterns specify. Your job is to convert discovered patterns into canonical structured artifacts---typed predicate definitions, action schemas with declared parameters, and Python grounding functions. 

\medskip
\textbf{User.}

\textsc{Current world model} (status-annotated as in stage~1):\quad \promptvar{current\_domain}

\textsc{Discovered patterns} (compile these faithfully):\quad \promptvar{patterns\_from\_stage\_1}

\medskip
\textsc{Task.}
\begin{enumerate}
\item Formalize each entry in \texttt{proposed\_actions} as a structured action schema with typed parameters, precondition, and effect.
\item Formalize each entry in \texttt{proposed\_predicates} as a typed predicate definition.
\item Apply each entry in \texttt{fix\_proposals} as a corrected definition with \texttt{intent: "fix"}.
\item Generate a Python grounding function for each new or fixed predicate.
\end{enumerate}

\medskip
\textsc{Output} (strict JSON):
\begin{verbatim}
{
  "new_predicates": [
    {"name": "...", "parameters": [...], "meaning": "...",
     "intent": "new"|"fix"}
  ],
  "new_actions": [
    {"name": "...", "parameters": [...],
     "precondition": <expr>, "effect": <expr>,
     "intent": "new"|"fix"}
  ],
  "grounding_code": "Python source string"
}
\end{verbatim}
\end{promptbox}

\begin{promptbox}[Propose-problem prompt]
\textbf{System.}

You are an expert in designing problems for an autonomous agent that builds and validates symbolic world models (PDDL domains).

Generate a single PDDL problem that serves the agent's current strategic intent. A problem can serve any of the following purposes simultaneously:
\begin{itemize}
\item \textbf{Validate} a hypothesis about a predicate or action (diagnostic / counterfactual).
\item \textbf{Build} a skill (navigation, manipulation, etc.).
\item \textbf{Discover} new symbolic structure.
\end{itemize}

\medskip
\textbf{User.}

\textsc{Strategic intent:}\quad \promptvar{strategic\_intent}

\textsc{Focus symbols} (exercise these):\quad \promptvar{focus\_symbols}

\textsc{Current domain:}\quad \promptvar{current\_domain\_pddl}

\textsc{Environment objects} (pick from this list):\quad \promptvar{available\_objects}

\medskip
\textsc{Guidelines.}
\begin{itemize}
\item Use only object names from the environment list and only values from typed constants. Do not invent names.
\item Include only objects participating in the goal.
\item Express \texttt{goal} as a logical formula over predicates already declared in the domain, using \texttt{atom}, \texttt{and}, \texttt{or}, \texttt{not}, or \texttt{seq} nodes.
\end{itemize}

\medskip
\textsc{Output} (strict JSON):
\begin{verbatim}
{
  "goal_name":   "...",
  "goal_description": "...",
  "why_useful":  "why this problem helps the strategic intent",
  "objects":     [{"name": "...", "type": "..."}],
  "init":        [],
  "goal":        <goal_expr>,
  "focus_symbols_exercised": [...],
  "success_criteria": "..."
}
\end{verbatim}
\end{promptbox}

\subsection{Baseline}\label{app:baselines}

\textbf{PPO.}
We use Proximal Policy Optimization~\citep{schulman2017proximal} with intrinsic motivation as a model-free continual RL baseline.
The agent is trained with both extrinsic task reward and intrinsic novelty bonus from Random Network Distillation~\citep[RND;][]{burda2018exploration}, which rewards visiting states whose features the predictor network cannot yet match. For intrinsic reward selection, to make a fair comparison, we varied three conditions: no intrinsic reward, count-based and RND-based intrinsic rewards. We showed the results of RND-based intrinsic rewards as it performs best across three conditions for PPO. 

The policy network is a three-layer CNN (channels 16-32-64, kernel size 2) followed by an \texttt{AdaptiveAvgPool2d}$(4\times4)$ that produces a fixed 1024-dimensional feature regardless of input resolution, then a 64-unit fully connected layer feeding separate actor and critic heads.
For both BabyAI and EMPA, the observation is the symbolic full-observation grid (object type, color index, state), with per-channel normalisation matching the integer ranges of each field.
All PPO hyperparameters follow standard settings: clipped surrogate objective ($\epsilon=.2$), GAE advantage estimation ($\gamma=.99$, $\lambda=.95$), 4 optimisation epochs per rollout batch, and linear learning-rate annealing within each environment phase.

\textbf{Motif.}
Motif~\citep{klissarov2023motif} constructs an intrinsic reward by eliciting trajectory preferences from an LLM and training a reward model on those preferences.
The pipeline runs four steps per curriculum environment and follows the same sequential continual-learning structure as PPO:
\begin{enumerate}[leftmargin=1.5em,itemsep=0.1em]
    \item \emph{Collect.}  A fresh RND agent (independent of the Motif policy) explores the current environment and saves the resulting episodes.
    \item \emph{Label.} $5000$ trajectory pairs are sampled from those episodes. Each pair is converted to a text caption---listing the objects the agent sees, their distances, and the chosen action at each step---and sent to an LLM.  The LLM returns a binary preference ($A > B$, $B > A$, or $A = B$) without access to the mission string.
    \item \emph{Train reward model.}  An MLP reward model (3-layer CNN encoder matching the \texttt{IntrinsicAgent} backbone, 512-unit hidden layer, scalar output head) is trained from scratch on the labeled pairs using a Bradley-Terry preference loss.
    \item \emph{RL with intrinsic reward.}  The \texttt{IntrinsicAgent}---a 3-layer CNN (32-64-64 channels, $3\times3$ kernels) with \texttt{AdaptiveAvgPool2d}$(4\times4)$ and 512-unit FC layer, feeding separate extrinsic and intrinsic critic heads---is trained with PPO.  The reward model's scalar output is used as the intrinsic bonus, together with extrinsic task reward, for the training signal (mirroring PPO's intrinsic-only setup).
\end{enumerate}

\textbf{Why Motif fails in the autotelic setting.}
Motif was originally demonstrated on Crafter, where the LLM can judge behavioral quality from goal-agnostic signals such as health loss, resource gain, and crafting events without knowing the specific task.
Our setting withholds the mission during training, since an agent that generalizes across open-ended task distributions must not be hard-wired to any particular goal~\citep{colas2022autotelic}, and knowledge acquired through task-specific reward shaping is goal-indexed rather than compositional.
Motif's labeler therefore has no basis for ranking one trajectory above another, because every primitive action---turn left, pick up, do nothing---is potentially correct depending on a goal it cannot see.
Its reward model achieves $50.2\%$ preference-prediction accuracy on a held-out validation split, indistinguishable from chance, so the intrinsic bonus it supplies carries no learning signal.

Inspecting the labeled pairs reveals the failure mode concretely. For example, in the \texttt{GoToLocal} phase, both segments share an \emph{identical} observation; only the action differs:

\begin{quote}\small
\textbf{Segment A:} \textit{You see a yellow ball 3 steps left and 3 steps forward. You see a blue box 3 steps forward. You see a purple key 2 steps right and 2 steps forward. $[\ldots]$
\textbf{Action: do nothing.}}

\textbf{Segment B:} [identical observation] \textbf{Action: pick up the front object.}

\smallskip
\textbf{LLM:} \textit{``In Segment B, the agent attempts to pick up an object, which isn't necessary for a \textbf{goto task} and thus doesn't help. However, doing nothing (A) is preferable as it avoids an irrelevant action.''} $\;\Rightarrow\;$ \textbf{[Rank] A $>$ B}
\end{quote}

The LLM hallucinates this is a \emph{goto} task. If the actual mission were ``pick up the grey key,'' Segment B is the correct behaviour and the label is backwards.
Across the full labeled dataset, 78.8\% of LLM responses invoke task-specific language (``goal,'' ``goto task,'' ``mission'') despite having no access to the mission string, indicating that the model compensates for missing context by hallucinating task assumptions rather than reasoning from first principles.
The only genuinely task-agnostic signal it reliably extracts---detecting wall collisions---appears in 77.5\% of responses, but such events are too sparse to yield a meaningful reward gradient.
Preference labels over trajectories therefore cannot supply a task-agnostic training signal in this setting, because the behavior being ranked is better or worse only relative to a goal the labeler does not have.

\textbf{WorldCoder.}
We run WorldCoder~\citep{tang2024worldcoder} from its public release with no algorithmic changes. Our only modifications are (i) swapping the OpenAI backend for a local Ollama server serving \texttt{deepseek-r1:70b}, the same backbone used by the Curator and Motif, and (ii) wrapping the per-episode driver in a sequential-curriculum loop so that the step budget, evaluation cadence, and held-out splits are the ones used for PPO and Motif. WorldCoder reads the same symbolic observation the other methods receive rather than pixels, and plans over its own synthesized code using the environment's primitive action set, so no separate grounding step is needed. All remaining hyperparameters keep their released defaults.

The synthesized code, a transition function plus one reward function per mission string, is carried over verbatim between curriculum environments, and only the experience buffer and the \texttt{mission\_accomplished} set are cleared at the environment boundary. When a new environment introduces transitions inconsistent with the carried-over transition function, the LLM edits that file in place. WorldCoder has no library object that can prune or retire an element, so every prior code path persists unless the LLM happens to delete it during a repair. This is the mechanism behind the library-size comparison in \cref{sec:llm}.

\section{Experiment~3 Supplementary Results}\label{app:curator-results}

\begin{table}[t]
     \centering
     \caption{
        \textbf{Training-time LLM cost.}
        Total LLM usage during training for WorldCoder, the Curator ablation without empowerment-based scoring, and the full Curator. All LLM-based methods use \texttt{deepseek-r1:70b}. Counts are aggregated across random seeds and include LLM calls, prompt tokens, completion tokens, and total tokens. The full Curator reduces calls and total tokens relative to WorldCoder while achieving stronger transfer (\cref{tab:llm-combined}).
        }
     \label{tab:llm-cost}
     \scriptsize
     \begin{tabular}{llrrrr}
     \toprule
     Model & Env & LLM calls & Prompt tokens & Completion tokens & Total tokens \\
     \midrule
     WorldCoder & BabyAI & $814$ & $5{,}586{,}741$ & $910{,}037$ & $6{,}496{,}778$ \\
     Hoarder & BabyAI & $631$ & $4{,}827{,}384$ & $1{,}054{,}712$ & $5{,}882{,}096$ \\
     Curator & BabyAI & $608$ & $4{,}322{,}112$ & $822{,}196$ & $5{,}144{,}308$ \\
     \midrule
     WorldCoder & Zelda & $603$ & $3{,}586{,}278$ & $540{,}929$ & $4{,}127{,}207$ \\
     Hoarder & Zelda  & $368$ & $3{,}032{,}226$ & $604{,}643$ & $3{,}636{,}869$ \\
     Curator  & Zelda & $321$ & $2{,}485{,}482$ & $457{,}386$ & $2{,}942{,}868$ \\
     \bottomrule
     \end{tabular}
\end{table}

\subsection{Cross-stage transfer}\label{app:curator-backward}

\cref{fig:curator-backward-babyai} evaluates each method's carried representation at every training stage on every BabyAI environment. The dashed line in each panel marks the stage at which that environment finishes its own training, where values to the left are forward transfer (success on the environment before it has been trained, using the library accumulated from earlier environments) and values to the right are backward transfer (retention or retroactive improvement after later environments are added). Three patterns are worth noting.

On the hardest environment $\mathrm{env}_3$, the Curator reaches ${\sim}.30$ success before $\mathrm{env}_3$ training begins (at $+\mathrm{env}_2$), and finishes at ${\sim}.45$ after its own training stage. The library accumulated from $\mathrm{env}_0\text{-}\mathrm{env}_2$ already handles a substantial fraction of $\mathrm{env}_3$ zero-shot, because predicates and operators promoted at earlier environments (e.g., \texttt{navigate\_to}, \texttt{pickup}, \texttt{toggle}) compose into the longer action chains $\mathrm{env}_3$ requires. WorldCoder, by contrast, starts $\mathrm{env}_3$ near zero, because its mission-specific reward functions do not compose.

On environments $\mathrm{env}_0\text{-}\mathrm{env}_2$, the Curator preserves earlier-environment success after the dashed line and on $\mathrm{env}_2$ shows clear retroactive improvement, where predicates promoted at \texttt{UnlockLocal} (notably \texttt{unlocked(?door)}) tighten operator preconditions on \texttt{Pickup} and \texttt{OpenRedDoor}, so earlier environments become \emph{easier} as the curriculum progresses. WorldCoder and PPO show the textbook forgetting profile---peak at the matched training stage, then decay (e.g., on $\mathrm{env}_1$, WorldCoder peaks ${\sim}.69$ at $+\mathrm{env}_1$ and drops to ${\sim}.10$ by $+\mathrm{env}_3$ as later edits to the carried code break earlier transitions).

The empowerment-ablated Hoarder, which keeps every LLM-proposed element rather than curating, shows neither forward nor backward gains, tracking the Curator's curve at $\mathrm{env}_0$ briefly before flattening out and never recovering on later environments.

\cref{fig:curator-backward-zelda} shows the same patterns on Zelda. The Curator transfers forward into $\mathrm{env}_1$ (${\sim}.20$ before $\mathrm{env}_1$ training begins, climbing to ${\sim}.49$ at $+\mathrm{env}_1$ and staying at ${\sim}.43$ after $+\mathrm{env}_2$) and into $\mathrm{env}_2$, while WorldCoder peaks lower (${\sim}.36$ on $\mathrm{env}_1$) and decays. The Hoarder again fails on both axes.

\begin{figure}[t]
    \centering
    \includegraphics[width=\textwidth]{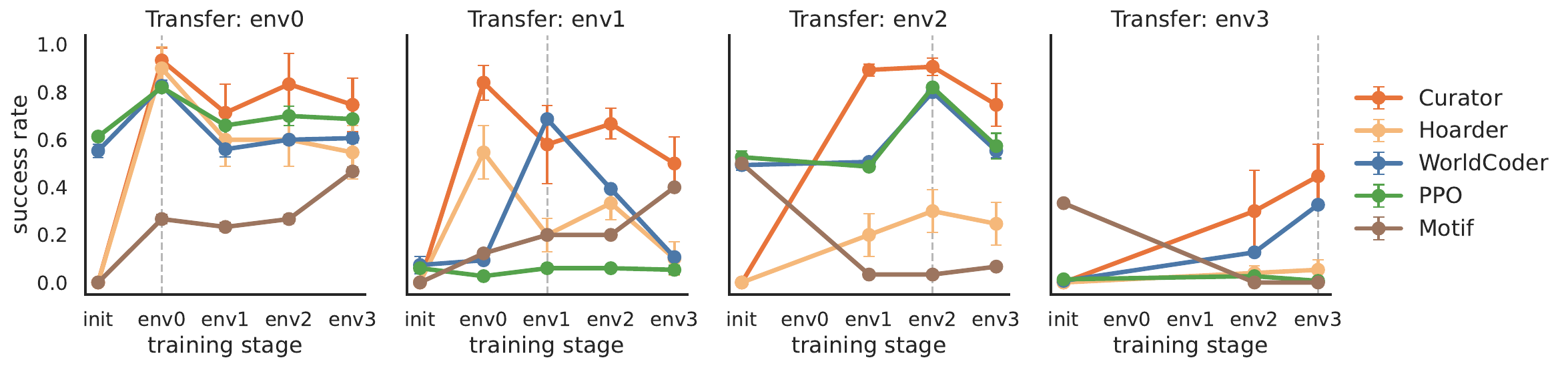}
    \caption{\textbf{BabyAI cross-stage transfer.} Each panel evaluates a fixed environment $\mathrm{env}_i$ as the curriculum proceeds. The dashed line marks the stage at which $\mathrm{env}_i$ finishes its own training; values to its left show forward transfer (success on $\mathrm{env}_i$ before it has been trained, achieved by the library accumulated from earlier environments), and values to its right show backward transfer (retention or retroactive improvement on $\mathrm{env}_i$ as later environments are added). The Curator transfers forward (e.g., $\mathrm{env}_3$ panel: $\sim .30$ at $+\mathrm{env}_2$ before $\mathrm{env}_3$ training begins) and backward (no decay after the dashed line; on $\mathrm{env}_3$, monotonic accumulation as $\mathrm{env}_3$ training itself proceeds). WorldCoder and PPO peak at each environment's matched stage and decay thereafter (the textbook forgetting profile). }
    \label{fig:curator-backward-babyai}
\end{figure}

\begin{figure}[t]
    \centering
    \includegraphics[width=0.76\textwidth]{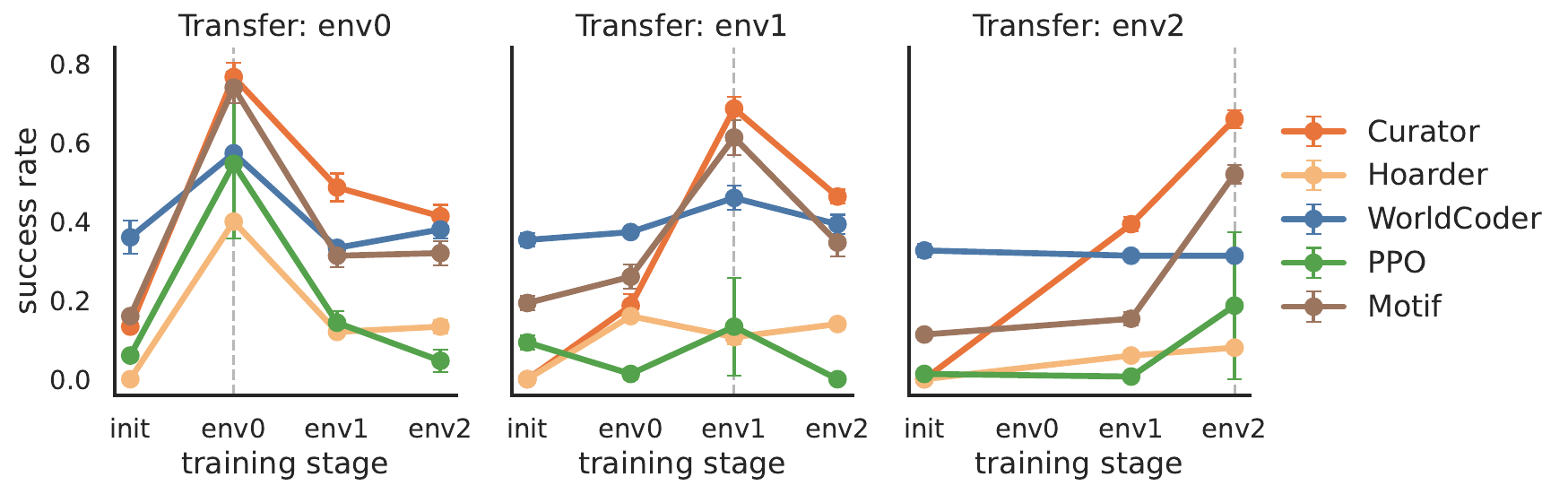}
    \caption{\textbf{Zelda cross-stage transfer.} Same protocol as \cref{fig:curator-backward-babyai}. The Curator transfers forward (visible on $\mathrm{env}_1$, $\mathrm{env}_2$ panels before each environment is trained) and preserves earlier-environment performance afterward, with retroactive improvement on $\mathrm{env}_1$. WorldCoder peaks lower than the Curator and decays; the Hoarder ablation again fails on both axes.}
    \label{fig:curator-backward-zelda}
\end{figure}

\subsection{Held-out transfer}\label{app:curator-heldout}

\cref{fig:curator-heldout-babyai,fig:curator-heldout-zelda} report zero-shot $\mathrm{reward\_success}$ of the frozen library $\Lib$ on configurations not seen during training. The frozen-library protocol is the same as forward and backward transfer (\cref{app:curator-protocol}): no training trajectories on the held-out environment, no candidate proposals, only library lookup and Actor execution.

On $\mathrm{env}_4$ (\texttt{PutNextLocal}), $\mathrm{env}_6$ (\texttt{PickupAbove}), and $\mathrm{env}_7$ (\texttt{Unlock} variant), the Curator reaches ${\sim}.45\text{-}.75$ success without ever having trained on these configurations, ahead of WorldCoder by $15\text{-}30$ absolute points and ahead of all policy baselines by an order of magnitude. The library composes, since predicates promoted at \texttt{UnlockLocal} (e.g., \texttt{unlocked(?door)}, \texttt{is\_key\_for}) and operators learned at \texttt{Pickup}/\texttt{OpenRedDoor} re-bind to new objects and layouts in the held-out environments without re-synthesis.

On $\mathrm{env}_5$ (\texttt{MoveTwoAcrossS5N2}), every method drops below $.15$, including the Curator. The environment requires achieving a conjunction of binary inter-object spatial relations (\texttt{next\_to(o1, o2)} and \texttt{next\_to(o3, o4)} simultaneously), but every training-environment goal is unary or agent-indexed (\texttt{holding}, \texttt{unlocked}, \texttt{at-location}). The bottleneck is the LLM proposer's goal diversity, since even when it generates a candidate \texttt{next\_to(?o1, ?o2)} predicate during training, the witness goals $\mathcal{W}_t$ available on training environments never test compositional object-object relations, so the predicate fails the empowerment threshold and is not promoted. The criterion behaves correctly given the probes it sees; the probes simply do not cover the structurally distinct goals $\mathrm{env}_5$ demands. We read this as a limitation of the proposer's goal-generation diversity, not of the empowerment criterion itself.

On both held-out Zelda environments, the Curator reaches ${\sim}.48$ success and outperforms WorldCoder by $8\text{-}12$ absolute points; policy baselines (PPO, Motif) remain near floor, and Hoarder collapses to ${\sim}.06$. The smaller margin reflects that Zelda's training-and-held-out gap is narrower (layout differences with the same primitive dynamics), so WorldCoder's mission-keyed reward functions transfer better than they do on BabyAI's compositionally distinct held-out environments.

Hoarder, which keeps every LLM-proposed element rather than curating, drops to ${\sim}.03\text{-}.06$ on every held-out environment in both families. The Curator's held-out gains therefore follow from the empowerment-based selection that keeps $\Lib$ compact and broadly applicable, rather than from the proposer generating useful candidates.

\begin{figure}[t]
    \centering
    \includegraphics[width=\textwidth]{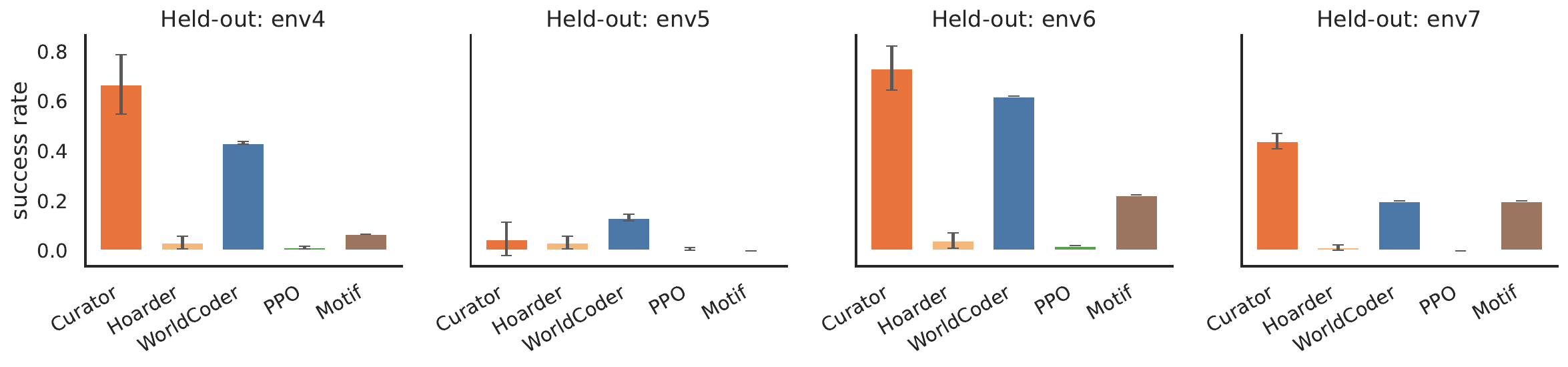}
    \caption{\textbf{BabyAI held-out transfer.} Zero-shot $\mathrm{reward\_success}$ of the frozen library $\Lib$ on four held-out environments, with error bars giving the standard deviation over 3 seeds.}
    \label{fig:curator-heldout-babyai}
\end{figure}

\begin{figure}[t]
    \centering
    \includegraphics[width=0.6\textwidth]{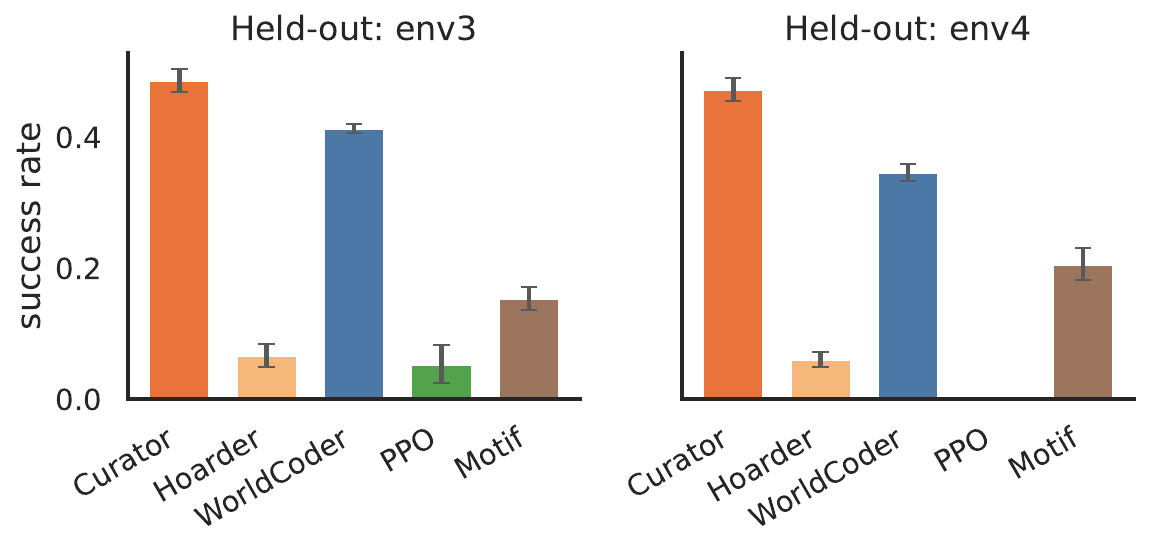}
    \caption{\textbf{Zelda held-out transfer.} Same protocol as \cref{fig:curator-heldout-babyai}, on the two held-out Zelda environments.}
    \label{fig:curator-heldout-zelda}
\end{figure}

\subsection{Library compactness and planning load}\label{app:curator-compactness}

The main-text \cref{tab:llm-combined} reports the curriculum-averaged library footprint behind the bloat claim in \cref{sec:llm}. This subsection provides the per-environment decomposition and the methodological details of the measurement.
We measure the carried-forward representation in characters of source code, broken into its two natural components on each side. WorldCoder carries a world model (Python transition function) plus a reward model containing one Python reward function per mission seen so far. The Curator carries the final PDDL domain file the planner compiles against, plus the \emph{grounding code} attached to each promoted predicate. 

\textbf{Per-environment breakdown.}
\cref{tab:bloat-summary} reports both decompositions per environment, mean over seeds. Two patterns are consistent across the families.

\emph{WorldCoder.} Its reward model grows monotonically across the four training environments, taking the reward-side character count from $6.2$\,KB to $20.4$\,KB ($3.3\times$). The world-model side stays around $5$\,KB throughout, since the transition function is edited in place rather than appended to. Total carried representation grows from $11.2$ to $24.7$\,KB.

\emph{Curator.} Its carried representation stays small and roughly flat. Across the EMPA/Zelda environments the symbolic side is $2$-$3$\,KB and the grounding code is $2$-$4$\,KB, for a total of $2$-$6$\,KB per environment---about $4$-$10\times$ smaller than WorldCoder's BabyAI reward bundle and approximately constant in the curriculum step. The trusted-element counts behind those byte numbers are $5$-$6$ predicates and $0$-$3$ operator schemas per environment.

\begin{table}[t]
    \centering
    \scriptsize
    \caption{\textbf{Carried-forward representation size across training environments.}
        Representation size is measured as source-code characters retained after each environment. WorldCoder size includes its generated world-model and reward code; Curator size includes the PDDL domain and predicate-grounding code used by the planner. Lower totals indicate a more compact carried-forward representation.
        }
    \label{tab:bloat-summary}
    \begin{tabular}{llrrrrrrrrr}
        \toprule
        & & \multicolumn{3}{c}{WorldCoder} & & \multicolumn{3}{c}{Curator} \\
        \cmidrule(lr){3-5}\cmidrule(lr){7-9}
        Family & Environment & WM.py & rwd.py & total & & PDDL & grnd.py & total\\
        \midrule
        BabyAI & env0 (GoToLocal)               & 4\,982 &  6\,207 & 11\,189 & & 2\,811 & 3\,883 & 6\,694\\
        BabyAI & env1 (Pickup)     & 4\,449 & 12\,305 & 16\,754 & & 2\,811 & 3\,883 & 6\,694\\
        BabyAI & env2 (OpenRedDoor)             & 5\,218 & 15\,443 & 20\,661 & & 5\,475 & 4\,091 & 9\,566\\
        BabyAI & env3 (UnlockLocal)             & 4\,337 & 20\,360 & 24\,697 & & 5\,141 & 4\,091 & 9\,232\\
        \midrule
        Zelda  & env0 (\texttt{empa\_250})      & 2\,799 & 1\,169 &  3\,969 & & 2\,173 & 2\,109 & 4\,283\\
        Zelda  & env1 (\texttt{empa\_251})      & 5\,906 & 2\,686 &  8\,592 & & 2\,275 & 2\,226 & 4\,501\\
        Zelda  & env2 (\texttt{empa\_252})      & 4\,986 & 3\,871 &  8\,857 & & 2\,572 & 1\,893 & 4\,465\\
        \bottomrule
    \end{tabular}
\end{table}

\end{document}